\documentclass[11pt]{article}

\usepackage{amsmath,amsbsy,amsfonts,amssymb,amsthm,dsfont,fullpage,units}
\usepackage{authblk,hyperref}
\usepackage[round]{natbib}
\usepackage{graphicx}
\def\E{\mathbb{E}}
\def\R{\mathbb{R}}
\def\t{{\scriptscriptstyle\top}}
\def\tl{\tilde}
\def\h{\hat}
\def\wh{\widehat}

\def\N{\mathcal{N}}

\newcommand\empar[1]{\smallskip\noindent\emph{#1}}

\newcommand\grad{\nabla}
\def\Sig{\varSigma}
\def\sig{\varsigma}

\def\eps{\epsilon}
\def\veps{\varepsilon}

\newcommand\alg{\textsc{LearnGMM}}

\newcommand\sign{\operatorname{sign}}
\newcommand\diag{\operatorname{diag}}
\newcommand\rank{\operatorname{rank}}
\newcommand\coherence{\operatorname{coherence}}
\newcommand\Span{\operatorname{span}}
\newcommand\range{\operatorname{range}}

\newcommand\mm{m}
\newcommand\mean{\mu}
\newcommand\Moment{\mathcal{M}}
\newcommand\Means{A}
\newcommand\Noise{\mathcal{N}}
\newcommand\ICA{\operatorname{ICA}}
\newcommand\GMM{\operatorname{GMM}}

\newcommand\dataset{\ensuremath{\mathcal{S}}}
\newcommand\Err{\ensuremath{\mathcal{E}}}
\newcommand\Bound{\ensuremath{\mathcal{B}}}
\newcommand\poly{\operatorname{poly}}

\newtheorem{lemma}{Lemma}
\newtheorem{theorem}{Theorem}
\theoremstyle{definition}
\newtheorem{condition}{Condition}

\theoremstyle{remark}
\newtheorem{remark}{Remark}

\newcommand\refprop[1]{Proposition~\ref{prop:#1}}
\newcommand\reflem[1]{Lemma~\ref{lem:#1}}
\newcommand\refthm[1]{Theorem~\ref{thm:#1}}
\newcommand\refcond[1]{Condition~\ref{cond:#1}}
\newcommand\refsec[1]{Section~\ref{section:#1}}
\newcommand\refapp[1]{Appendix~\ref{section:#1}}
\newcommand\reffig[1]{Figure~\ref{fig:#1}}

\newenvironment{thm2}[1]{\def\thetheorem{\ref{thm:#1} restated}\begin{theorem}}{\end{theorem}\addtocounter{theorem}{-1}}

\title{Learning mixtures of spherical Gaussians:
\mbox{moment methods and spectral decompositions}}
\author{Daniel Hsu}
\author{Sham M. Kakade}
\affil{Microsoft Research New England}

\begin{document}
\maketitle
{\def\thefootnote{}
\footnotetext{E-mail: \texttt{\{dahsu,skakade\}@microsoft.com}}}

\begin{abstract}
This work provides a computationally efficient and statistically consistent
moment-based estimator for mixtures of spherical Gaussians.
Under the condition that component means are in general position, a simple
spectral decomposition technique yields consistent parameter estimates from
low-order observable moments, without additional minimum separation
assumptions needed by previous computationally efficient estimation
procedures.
Thus computational and information-theoretic barriers to efficient
estimation in mixture models are precluded when the mixture components have
means in general position and spherical covariances.
Some connections are made to estimation problems related to independent
component analysis.
\end{abstract}

\section{Introduction}
\label{section:intro}

The Gaussian mixture model~\citep{Pearson94,TSM85} is one of the most
well-studied and widely-used models in applied statistics and machine
learning.
An important special case of this model (the primary focus of this work)
restricts the Gaussian components to have spherical covariance matrices;
this probabilistic model is closely related to the (non-probabilistic)
$k$-means clustering problem~\citep{kmeans}.

The mixture of spherical Gaussians model is specified as follows.
Let $w_i$ be the probability of choosing component $i \in [k] :=
\{1,2,\dotsc,k\}$, let $\mean_1, \mean_2, \dotsc, \mean_k \in \R^d$ be the
component mean vectors, and let $\sigma_1^2, \sigma_2^2, \dotsc, \sigma_k^2
\geq 0$ be the component variances.
Define
\begin{align*}
w & := [w_1,w_2,\dotsc,w_k]^\t \in \R^k , &
\Means & := [\mean_1|\mean_2|\dotsb|\mean_k] \in \R^{d \times k} ;
\end{align*}
so $w$ is a probability vector, and $\Means$ is the matrix whose columns
are the component means.
Let $x \in \R^k$ be the (observed) random vector given by
\begin{align*}
x & := \mean_h + z ,
\end{align*}
where $h$ is the discrete random variable with $\Pr(h = i) = w_i$ for $i
\in [k]$, and $z$ is a random vector whose conditional distribution given
$h = i$ (for some $i \in [k]$) is the multivariate Gaussian
$\N(0,\sigma_i^2 I)$ with mean zero and covariance $\sigma_i^2
I$.

The estimation task is to accurately recover the model parameters
(component means, variances, and mixing weights) $\{ (\mean_i, \sigma_i^2,
w_i) : i \in [k] \}$ from independent copies of $x$.

This work gives a procedure for efficiently and exactly recovering the
parameters using a simple spectral decomposition of low-order moments of
$x$, under the following condition.
\begin{condition}[Non-degeneracy] \label{cond:rank}
The component means span a $k$-dimensional subspace (\emph{i.e.}, the
matrix $\Means$ has column rank $k$),
and the vector $w$ has strictly positive entries.
\end{condition}
The proposed estimator is based on a spectral decomposition technique
\citep{Chang96,MR06,AHK12}, and is easily stated in terms of exact
population moments of the observed $x$.
With finite samples, one can use a plug-in estimator based on empirical
moments of $x$ in place of exact moments.
These empirical moments converge to the exact moments at a rate of
$O(n^{-1/2})$, where $n$ is the sample size.
As discussed in \refsec{discussion}, sample complexity bounds for accurate
parameter estimation can be derived using matrix perturbation arguments
\citep{AHK12}.
Since only low-order moments are required by the plug-in estimator, the
sample complexity is polynomial in the relevant parameters of the
estimation problem.

\paragraph{Related work.}

The first estimators for the Gaussian mixture models were based on the
method-of-moments, as introduced by \citet{Pearson94} \citep[see also][and
the references therein]{LB93}.
Roughly speaking, these estimators are based on finding parameters under
which the Gaussian mixture distribution has moments approximately matching
the observed empirical moments.
Finding these parameters typically involves solving systems of multivariate
polynomial equations, which is typically computationally challenging.
Besides this, the order of the moments of some of the early moment-based
estimators were either growing with the dimension $d$ or the number of
components $k$, which is undesirable because the empirical estimates of
such high-order moments may only be reliable when the sample size is
exponential in $d$ or $k$.
Both the computational and sample complexity issues have been addressed in
recent years, at least under various restrictions.
For instance, several distance-based estimators require that the component
means be well-separated in Euclidean space, by at least some large factor
times the directional standard deviation of the individual component
distributions~\citep{Das99,AK01,DS07,VW02,CR08}, but otherwise have
polynomial computational and sample complexity.
Some recent moment-based estimators avoid the minimum separation condition
of distance-based estimators by requiring either computational or data
resources exponential in the number of mixing components $k$ (but not the
dimension $d$)~\citep{BS10,KMV10,MV10} or by making a non-degenerate
multi-view assumption~\citep{AHK12}.

By contrast, the moment-based estimator described in this work does not
require a minimum separation condition, exponential computational or data
resources, or non-degenerate multiple views.
Instead, it relies only on the non-degeneracy condition discussed above
together with a spherical noise condition.
The non-degeneracy condition is much weaker than an explicit minimum
separation condition because the parameters can be arbitrarily close to
being degenerate, as long as the sample size grows polynomially with a
natural quantity measuring this closeness to degeneracy (akin to a
condition number).
Like other moment-based estimators, the proposed estimator is based on
solving multivariate polynomial equations, although these solutions can be
found efficiently because the problems are cast as eigenvalue
decompositions of symmetric matrices, which are efficient to compute.

Recent work of \citet{MV10} demonstrates an information-theoretic barrier
to estimation for general Gaussian mixture models.
More precisely, they construct a pair of one-dimensional mixtures of
Gaussians (with separated component means) such that the statistical
distance between the two mixture distributions is exponentially small in
the number of components.
This implies that in the worst case, the sample size required to obtain
accurate parameter estimates must grow exponentially with the number of
components, even when the component distributions are non-negligibly
separated.
A consequence of the present work is that natural non-degeneracy conditions
preclude these worst case scenarios.
The non-degeneracy condition in this work is similar to one used for
bypassing computational (cryptographic) barriers to estimation for hidden
Markov models~\citep{Chang96,MR06,HKZ12,AHK12}.

Finally, it is interesting to note that similar algebraic techniques have
been developed for certain models in independent component analysis (ICA)
\citep{Comon94,CC96,HO00,Comon:book,AroraICA} and other closely related
problems \citep{FJK96,NR09}.
In contrast to the ICA setting, handling non-spherical Gaussian noise for
mixture models appears to be a more delicate issue.
These connections and open problems are further discussed in
\refsec{discussion}.

\section{Moment-based estimation}
\label{section:gmm}

This section describes a method-of-moments estimator for the spherical
Gaussian mixture model.

The following theorem is the main structural result that relates the model
parameters to observable moments.
\begin{theorem}[Observable moment structure] \label{thm:moments} 
Assume \refcond{rank} holds.
The average variance $\bar\sigma^2 := \sum_{i=1}^k w_i \sigma_i^2$ is the
smallest eigenvalue of the covariance matrix $\E[ (x - \E[x]) (x -
\E[x])^\t ]$.
Let $v \in \R^d$ be any unit norm eigenvector corresponding to the
eigenvalue $\bar\sigma^2$.
Define
\begin{align*}
M_1 & := \E[ x (v^\t (x - \E[x]))^2 ] \in \R^d , \\
M_2 & := \E[ x \otimes x ] - \bar\sigma^2 I \in \R^{d \times d} , \\
M_3 & := \E[ x \otimes x \otimes x ] -
\sum_{i=1}^d \bigl( M_1 \otimes e_i \otimes e_i
+ e_i \otimes M_1 \otimes e_i
+ e_i \otimes e_i \otimes M_1 \bigr) \in \R^{d \times d \times d}
\end{align*}
(where $\otimes$ denotes tensor product, and $\{ e_1, e_2, \dotsc, e_d \}$
is the coordinate basis for $\R^d$).
Then
\begin{equation*}
M_1 = \sum_{i=1}^k w_i \ \sigma_i^2 \mean_i , \qquad
M_2 = \sum_{i=1}^k w_i \ \mean_i \otimes \mean_i , \qquad
M_3 = \sum_{i=1}^k w_i \ \mean_i \otimes \mean_i \otimes \mean_i .
\end{equation*}
\end{theorem}
\begin{remark}
We note that in the special case where $\sigma_1^2 = \sigma_2^2 = \dotsb =
\sigma_k^2 = \sigma^2$ (\emph{i.e.}, the mixture components share a common
spherical covariance matrix), the average variance $\bar\sigma^2$ is simply
$\sigma^2$, and $M_3$ has a simpler form:
\[
M_3 = \E[ x \otimes x \otimes x ] -
\sigma^2 \sum_{i=1}^d \bigl( \E[x] \otimes e_i \otimes e_i
+ e_i \otimes \E[x] \otimes e_i
+ e_i \otimes e_i \otimes \E[x] \bigr) \in \R^{d \times d \times d}
.
\]
There is no need to refer to the eigenvectors of the covariance matrix or
$M_1$.
\end{remark}
\begin{proof}[Proof of \refthm{moments}]
We first characterize the smallest eigenvalue of the covariance matrix of
$x$, as well as all corresponding eigenvectors $v$.
Let $\bar\mean := \E[x] = \E[\mean_h] = \sum_{i=1}^k w_i \mean_i$.
The covariance matrix of $x$ is
\begin{align*}
\E[ (x - \bar\mean) \otimes (x - \bar\mean) ]
& = \sum_{i=1}^k w_i \
\biggl( (\mean_i - \bar\mean) \otimes (\mean_i - \bar\mean)
+ \sigma_i^2 I \biggr)
\\
& = \sum_{i=1}^k w_i \
(\mean_i - \bar\mean) \otimes (\mean_i - \bar\mean)
\ + \ \bar\sigma^2 I
.
\end{align*}
Since the vectors $\mean_i - \bar\mean$ for $i \in [k]$ are linearly
dependent ($\sum_{i=1}^k w_i (\mean_i - \bar\mean) = 0$), the positive
semidefinite matrix $\sum_{i=1}^k w_i (\mean_i - \bar\mean) \otimes
(\mean_i - \bar\mean)$ has rank $r \leq k-1$.
Thus, the $d-r$ smallest eigenvalues are exactly $\bar\sigma^2$, while all
other eigenvalues are strictly larger than $\bar\sigma^2$.
The strict separation of eigenvalues implies that every eigenvector
corresponding to $\bar\sigma^2$ is in the null space of $\sum_{i=1}^k w_i
(\mean_i - \bar\mean) \otimes (\mean_i - \bar\mean)$; thus $v^\t (\mean_i -
\bar\mean) = 0$ for all $i \in [k]$.

Now we can express $M_1$, $M_2$, and $M_3$ in terms of the parameters
$w_i$, $\mean_i$, and $\sigma_i^2$.
First,
\begin{align*}
M_1 & = \E[ x (v^\t (x - \E[x]))^2 ]
= \E[ (\mean_h + z) (v^\t (\mean_h - \bar\mean + z))^2 ]
= \E[ (\mean_h + z) (v^\t z)^2 ]
= \E[ \mean_h \sigma_h^2 ]
,
\end{align*}
where the last step uses the fact that $z|h \sim \N(0,\sigma_h^2 I)$, which
implies that conditioned on $h$, $\E[ (v^\t z)^2 | h] = \sigma_h^2$ and
$\E[ z (v^\t z)^2 |h ] = 0$.
Next, observe that $\E[ z \otimes z ] = \sum_{i=1}^k w_i \sigma_i^2 I =
\bar\sigma^2 I$, so
\begin{align*}
M_2 & = \E[ x \otimes x ] - \bar\sigma^2 I
= \E[ \mean_h \otimes \mean_h ] + \E[ z \otimes z ] - \bar\sigma^2 I
= \E[ \mean_h \otimes \mean_h ]
= \sum_{i=1}^k w_i \ \mean_i \otimes \mean_i
.
\end{align*}
Finally, for $M_3$, we first observe that
\[
\E[ x \otimes x \otimes x ]
= \E[ \mean_h \otimes \mean_h \otimes \mean_h ]
+ \E[ \mean_h \otimes z \otimes z ]
+ \E[ z \otimes \mean_h \otimes z ]
+ \E[ z \otimes z \otimes \mean_h ]
\]
(terms such as $\E[ \mean_h \otimes \mean_h \otimes z ]$ and $\E[ z \otimes
z \otimes z ]$ vanish because $z|h \sim \N(0,\sigma_h^2 I)$).
We now claim that $\E[ \mean_h \otimes z \otimes z ] = \sum_{i=1}^d M_1
\otimes e_i \otimes e_i$.
This holds because
\begin{align*}
\E[ \mean_h \otimes z \otimes z ]
& = \E\bigl[ \E[ \mean_h \otimes z \otimes z | h ] \bigr]
\\
& = \E\biggl[
\E\biggl[ \sum_{i,j=1}^d z_i z_j \ \mean_h \otimes e_i \otimes e_j
\Big| h \biggr] \biggr]
= \E\biggl[
\sum_{i=1}^d \sigma_h^2 \ \mean_h \otimes e_i \otimes e_i \biggr]
= \sum_{i=1}^d M_1 \otimes e_i \otimes e_i
,
\end{align*}
crucially using the fact that $\E[ z_i z_j | h ] = 0$ for $i \neq j$ and
$\E[ z_i^2 | h ] = \sigma_h^2$.
By the same derivation, we have $\E[ z \otimes \mean_h \otimes z ] =
\sum_{i=1}^d e_i \otimes M_1 \otimes e_i$ and $\E[ z \otimes z \otimes
\mean_h ] = \sum_{i=1}^d e_i \otimes e_i \otimes M_1$.
Therefore,
\[ M_3
= \E[ x \otimes x \otimes x ]
- \bigl( \E[ \mean_h \otimes z \otimes z ]
+ \E[ z \otimes \mean_h \otimes z ]
+ \E[ z \otimes z \otimes \mean_h ]
\bigr)
= \E[ \mean_h \otimes \mean_h \otimes \mean_h ]
= \sum_{i=1}^k w_i \mean_i \otimes \mean_i \otimes \mean_i
\]
as claimed.
\end{proof}

\refthm{moments} shows the relationship between (some functions of) the
observable moments and the desired parameters.
A simple estimator based on this moment structure is given in the following
theorem.
For a third-order tensor $T \in \R^{d \times d \times d}$, we define
$T(\eta) := \sum_{i_1=1}^d \sum_{i_2=1}^d \sum_{i_3=1}^d T_{i_1,i_2,i_3}
\eta_{i_3} \ e_{i_1} \otimes e_{i_2} \in \R^{d \times d}$ for any vector
$\eta\in\R^d$.

\begin{theorem}[Moment-based estimator] \label{thm:spherical} 
The following can be added to the results of Theorem~\ref{thm:moments}.
Suppose $\eta^\t \mean_1, \eta^\t \mean_2, \dotsc, \eta^\t \mean_k$ are
distinct and non-zero (which is satisfied almost surely, for instance, if
$\eta$ is chosen uniformly at random from the unit sphere in $\R^d$).
Then the matrix
\begin{align*}
M_{\GMM}(\eta) & := M_2^{\dag1/2} M_3(\eta) M_2^{\dag1/2}
\end{align*}
is diagonalizable (where ${}^\dag$ denotes the Moore-Penrose
pseudoinverse); its non-zero eigenvalue / eigenvector pairs
$(\lambda_1,v_1), (\lambda_2,v_2), \dotsc, (\lambda_k,v_k)$ satisfy
$\lambda_i = \eta^\t \mean_{\pi(i)}$ and $M_2^{1/2} v_i = s_i
\sqrt{w_{\pi(i)}} \mean_{\pi(i)}$ for some permutation $\pi$ on $[k]$ and
signs $s_1, s_2, \dotsc, s_k \in \{\pm1\}$.
The $\mean_i$, $\sigma_i^2$, and $w_i$ are recovered (up to permutation)
with 
\begin{equation*}
\mean_{\pi(i)}
= \frac{\lambda_i}{\eta^\t M_2^{1/2} v_i} M_2^{1/2} v_i
, \qquad
\sigma_i^2 = \frac{1}{w_i} e_i^\t \Means^{\dag} M_1
, \qquad
w_i = e_i^\t \Means^{\dag} \E[x]
.
\end{equation*}
\end{theorem}
\begin{proof}
By Theorem~\ref{thm:moments},
\begin{align*}
M_1 & = \Means \diag(\sigma_1^2, \sigma_2^2, \dotsc, \sigma_k^2) w , &
M_2 & = \Means \diag(w) \Means^\t , &
M_3(\eta) & = \Means \diag(w) D_1(\eta) \Means^\t ,
\end{align*}
where
$D_1(\eta)
:= \diag(\eta^\t \mean_1, \eta^\t \mean_2, \dotsc, \eta^\t \mean_k)$.

Let $USR^\t$ be the thin SVD of $\Means \diag(w)^{1/2}$ ($U \in \R^{d \times
k}$, $S \in \R^{k \times k}$, and $R \in \R^{k \times k}$), so $M_2 =
US^2U^\t$ and $M_2^{\dag1/2} = US^{-1}U^\t$ since $\Means \diag(w)^{1/2}$
has rank $k$ by assumption.
Also by assumption, the diagonal entries of $D_1(\eta)$ are distinct and
non-zero.
Therefore, every non-zero eigenvalue of the symmetric matrix
$M_{\GMM}(\eta) = UR^\t D_1(\eta) RU^\t$ has geometric multiplicity one.
Indeed, these non-zero eigenvalues $\lambda_i$ are the diagonal entries of
$D_1(\eta)$ (up to some permutation $\pi$ on $[k]$), and the corresponding
eigenvectors $v_i$ are the columns of $UR^\t$ up to signs:
\begin{equation*}
\lambda_i = \eta^\t \mean_{\pi(i)} \quad \text{and} \quad
v_i = s_i UR^\t e_{\pi(i)} .
\end{equation*}
Now, since
\begin{equation*}
M_2^{1/2} v_i = s_i \sqrt{w_{\pi(i)}} \mean_{\pi(i)} , \qquad
\frac{\lambda_i}{\eta^\t M_2^{1/2} v_i}
= \frac{\eta^\t \mean_{\pi(i)}}{s_i \sqrt{w_{\pi(i)}}
\eta^\t \mean_{\pi(i)}}
= \frac{1}{s_i \sqrt{w_{\pi(i)}}} ,
\end{equation*}
it follows that
\[
\mean_{\pi(i)} = \frac{\lambda_i}{\eta^\t M_2^{1/2} v_i} M_2^{1/2} v_i ,
\quad i \in [k] .
\]
The claims regarding $\sigma_i^2$ and $w_i$ are also evident from the
structure of $M_1$ and $\E[x] = \Means w$.
\end{proof}

An efficiently computable plug-in estimator can be derived from
\refthm{spherical}.
We state one such algorithm (called \alg) in \refapp{finite}; for
simplicity, we restrict to the case where the components share the same
common spherical covariance, \emph{i.e.}, $\sigma_1^2 = \sigma_2^2 = \dotsb
= \sigma_k^2 = \sigma^2$.
The following theorem provides a sample complexity bound for accurate
estimation of the component means.
Since only low-order moments are used, the sample complexity is polynomial
in the relevant parameters of the estimation problem (in particular, the
dimension $d$ and the number of mixing components $k$).
It is worth noting that the polynomial is quadratic in the inverse accuracy
parameter $1/\veps$; this owes to the fact that the empirical moments
converge to the population moments at the usual $n^{-1/2}$ rate as per the
central limit theorem.

\begin{theorem}[Finite sample bound] \label{thm:finite}
There exists a polynomial $\poly(\cdot)$ such that the following holds.
Let $M_2$ be the matrix defined in \refthm{spherical}, and $\sig_t[M_2]$
be its $t$-th largest singular value (for $t \in [k]$).
Let
$b_{\max} := \max_{i \in [k]} \|\mean_i\|_2$ and
$w_{\min} := \min_{i \in [k]} w_i$.
Pick any $\veps, \delta \in (0,1)$.
Suppose the sample size $n$ satisfies
\begin{align*}
n & \geq
\poly\Bigl(
d,k,1/\veps,\log(1/\delta),
1/w_{\min},
\sig_1[M_2]/\sig_k[M_2],
b_{\max}^2/\sig_k[M_2],
\sigma^2/\sig_k[M_2],
\Bigr)
.
\end{align*}
Then with probability at least $1-\delta$ over the random sample and the
internal randomness of the algorithm, there exists a permutation $\pi$ on
$[k]$ such that the $\{ \h\mean_i : i \in [k] \}$ returned by
\alg\ satisfy
\begin{align*}
\|\h\mean_{\pi(i)} - \mean_i\|_2
& \leq \Bigl( \|\mean_i\|_2 + \sqrt{\sig_1[M_2]} \Bigr) \veps
\end{align*}
for all $i \in [k]$.
\end{theorem}
The proof of \refthm{finite} is given in \refapp{finite}.
It is also easy to obtain accuracy guarantees for estimating $\sigma^2$ and
$w$.
The role of \refcond{rank} enters by observing that $\sig_k[M_2] = 0$ if
either $\rank(\Means) < k$ or $w_{\min} = 0$, as $M_2 = \Means \diag(w)
\Means^\t$.
The sample complexity bound then becomes trivial in this case, as the bound
grows with $1/\sig_k[M_2]$ and $1/w_{\min}$.
Finally, we also note that \alg\ is just one (easy to state) way to obtain
an efficient algorithm based on the structure in \refthm{moments}.
It is also possible to use, for instance, simultaneous diagonalization
techniques~\citep{BBM93} or orthogonal tensor decompositions~\citep{Tensor}
to extract the parameters from (estimates of) $M_2$ and $M_3$; these
alternative methods are more robust to sampling error, and are therefore
recommended for practical implementation.

\section{Discussion}
\label{section:discussion}

\paragraph{Multi-view methods and a simpler algorithm in higher
dimensions.}

Some previous work of the authors on moment-based estimators for the
Gaussian mixture model relies on a non-degenerate multi-view
assumption~\citep{AHK12}.
In this work, it is shown that if each mixture component $i$ has an
axis-aligned covariance $\Sig_i := \diag(\sigma_{1,i}^2, \sigma_{2,i}^2,
\dotsc, \sigma_{d,i}^2)$, then under some additional mild assumptions
(which ultimately require $d > k$), a moment-based method can be used to
estimate the model parameters.
The idea is to partition the coordinates $[d]$ into three groups, inducing
multiple ``views'' $x = (x_1,x_2,x_3)$ with each $x_t \in \R^{d_t}$ for
some $d_t \geq k$ such that $x_1$, $x_2$, and $x_3$ are conditionally
independent given $h$.
When the matrix of conditional means $A_t := [ \E[x_t | h = 1] | \E[x_t | h
= 2 ] | \dotsb | \E[x_t | h = k] ] \in \R^{d_t \times k}$ for each view $t
\in \{1,2,3\}$ has rank $k$, then an efficient technique similar to that
described in \refthm{spherical} will recover the parameters.
Therefore, the problem is reduced to partitioning the coordinates so that
the resulting matrices $A_t$ have rank $k$.

In the case where each component covariance is spherical ($\Sig_i =
\sigma_i^2 I$), we may simply apply a random rotation to $x$ before
(arbitrarily) splitting into the three views.
Let $\tl{x} := \Theta x$ for a random orthogonal matrix $\Theta \in \R^{d
\times d}$, and partition the coordinates so that $\tl{x} = (\tl{x}_1,
\tl{x}_2, \tl{x}_3)$ with $\tl{x}_t \in \R^{d_t}$ and $d_t \geq k$.
By the rotational invariance of the multivariate Gaussian distribution, the
distribution of $\tl{x}$ is still a mixture of spherical Gaussians, and
moreover, the matrix of conditional means $\tl{A}_t := [ \E[\tl{x}_t | h =
1] | \E[\tl{x}_t | h = 2 ] | \dotsb | \E[\tl{x}_t | h = k] ] \in \R^{d_t
\times k}$ for each view $\tl{x}_t$ has rank $k$ with probability $1$.
To see this, observe that a random rotation in $\R^d$ followed by a
restriction to $d_t$ coordinates is simply a random projection from $\R^d$
to $\R^{d_t}$, and that a random projection of a linear subspace of
dimension $k$ (in particular, the range of $A$) to $\R^{d_t}$ is almost
surely injective as long as $d_t \geq k$.
Therefore it is sufficient to require $d \geq 3k$ so that it is possible to
split $\tl{x}$ into three views, each of dimension $d_t \geq k$.
To guarantee that the $k$-th largest singular value of each $\tl{A}_t$ is
bounded below in terms of the $k$-th largest singular value of $A$ (with
high probability), we may require $d$ to be somewhat larger: $O(k \log k)$
certainly works (see \refapp{incoherence}), and we conjecture $c \cdot k$
for some $c > 3$ is in fact sufficient.

\paragraph{Spectral decomposition approaches for ICA.}

The Gaussian mixture model shares some similarities to a standard
model for independent component analysis
(ICA)~\citep{Comon94,CC96,HO00,Comon:book}.
Here, let $h\in \R^k$ be a random vector with independent entries, and let
$z \in \R^k$ be multivariate Gaussian random vector.
We think of $h$ as an unobserved signal and $z$ as noise.
The observed random vector is
\begin{align*}
x & := \Means h + z
\end{align*}
for some $\Means \in \R^{k \times k}$, where $h$ and $z$ are assumed to be
independent.
(For simplicity, we only consider square $\Means$, although it is easy to
generalize to $\Means \in \R^{d \times k}$ for $d \geq k$.)

In contrast to this ICA model, the spherical Gaussian mixture model is one
where $h$ would take values in $\{e_1,e_2,\dotsc,e_k\}$, and the covariance
of $z$ (given $h$) is spherical.

For ICA, a spectral decomposition approach related to the one described
in \refthm{spherical} can be used to estimate the columns of $\Means$ (up to
scale), without knowing the noise covariance $\E[zz^\t]$.
Such an estimator can be obtained from \refthm{ica} using techniques
commonplace in the ICA literature; its proof is given in \refapp{ica} for
completeness.

\begin{theorem} \label{thm:ica}
In the ICA model described above, assume $\E[h_i] = 0$, $\E[h_i^2] = 1$,
and $\kappa_i := \E[h_i^4] - 3 \neq 0$ (\emph{i.e.}, the excess kurtosis is
non-zero), and that $\Means$ is non-singular.
Define $f \colon \R^k \to \R$ by
\begin{align*}
f(\eta) & := {12}^{-1} \bigl( \mm_4(\eta) - 3 \mm_2(\eta)^2 \bigr)
\end{align*}
where $\mm_p(\eta) := \E[ (\eta^\t x)^p ]$.
Suppose $\phi \in \R^k$ and $\psi \in \R^k$ are such that $\frac{(\phi^\t
\mean_1)^2}{(\psi^\t \mean_1)^2}, \frac{(\phi^\t \mean_2)^2}{(\psi^\t
\mean_2)^2}, \dotsc, \frac{(\phi^\t \mean_k)^2}{(\psi^\t \mean_k)^2} \in
\R$ are distinct.
Then the matrix
\begin{align*}
M_{\ICA}(\phi,\psi)
& := \bigl( \grad^2f(\phi) \bigr) \bigl( \grad^2f(\psi) \bigr)^{-1}
\end{align*}
is diagonalizable; the eigenvalues are $\frac{(\phi^\t \mean_1)^2}{(\psi^\t
\mean_1)^2}, \frac{(\phi^\t \mean_2)^2}{(\psi^\t \mean_2)^2}, \dotsc,
\frac{(\phi^\t \mean_k)^2}{(\psi^\t \mean_k)^2}$ and each have geometric
multiplicity one, and the corresponding eigenvectors are $\mean_1, \mean_2,
\dotsc, \mean_k$ (up to scaling and permutation).
\end{theorem}

Again, choosing $\phi$ and $\psi$ as random unit vectors ensures the
distinctness assumption is satisfied almost surely, and a finite sample
analysis can be given using standard matrix perturbation
techniques~\citep{AHK12}.
A number of related deterministic algorithms based on algebraic techniques
are discussed in the text of \citet{Comon:book}.
Recent work of \citet{AroraICA} provides a finite sample complexity
analysis for an efficient estimator based on local search.

\paragraph{Non-degeneracy.}

The non-degeneracy assumption (\refcond{rank}) is quite natural, and its
has the virtue of permitting tractable and consistent estimators.
Although previous work has typically tied it with additional assumptions,
this work shows that they are largely unnecessary.

One drawback of \refcond{rank} is that it prevents the straightforward
application of these techniques to certain problem domains (\emph{e.g.},
automatic speech recognition (ASR), where the number of mixture components
is typically enormous, but the dimension of observations is relatively
small; alternatively, the span of the means has dimension $< k$).
To compensate, one may require multiple views, which are granted by a
number of models, including hidden Markov models used in ASR
\citep{HKZ12,AHK12}, and combining these views in a tensor product fashion
\citep{AMR09}.
This increases the complexity of the estimator, but that may be inevitable
as estimation for certain non-singular models is conjectured to be
computationally intractable \citep{MR06}.

\subsection*{Acknowledgements}
We thank Dean Foster and Anima Anandkumar for helpful insights.
We also thank Rong Ge and Sanjeev Arora for discussions regarding their
recent work on ICA.

\bibliographystyle{plainnat}
\bibliography{mog}

\appendix

\section{Connection to independent component analysis}
\label{section:ica}

\begin{proof}[Proof of \refthm{ica}]
It can be shown that
\begin{align*}
\mm_2(\eta) & = \E[ (\eta^\t \Means h)^2 ] + \E[ (\eta^\t z)^2 ] ,
& \mm_4(\eta) & = \E[ (\eta^\t \Means h)^4 ] - 3 \E[ (\eta^\t \Means h)^2
]^2 + 3\mm_2(\eta)^2 .
\end{align*}
By the assumptions,
\begin{align*}
\E[(\eta^\t \Means h)^4]
& = \sum_{i=1}^k (\eta^\t \mean_i)^4 \E[ h_i^4 ] + 3 \sum_{i \neq j}
(\eta^\t \mean_i)^2 (\eta^\t \mean_j)^2 \\
& = \sum_{i=1}^k \kappa_i (\eta^\t \mean_i)^4 + 3 \sum_{i,j} (\eta^\t
\mean_i)^2 (\eta^\t \mean_j)^2 \\
& = \sum_{i=1}^k \kappa_i (\eta^\t \mean_i)^4 + 3 \E[ (\eta^\t \Means h)^2 ]^2 ,
\end{align*}
and therefore
\begin{align*}
f(\eta)
& = {12}^{-1} \bigl(
\E[ (\eta^\t \Means h)^4 ] - 3 \E[ (\eta^\t \Means h)^2 ]^2
\bigr)
= {12}^{-1} \sum_{i=1}^k \kappa_i (\eta^\t \mean_i)^4 .
\end{align*}
The Hessian of $f$ is given by
\begin{align*}
\grad^2f(\eta)
& = \sum_{i=1}^k \kappa_i (\eta^\t \mean_i)^2 \mean_i\mean_i^\t .
\end{align*}
Define the diagonal matrices
\begin{align*}
K
& := \diag(\kappa_1,\kappa_2,\dotsc,\kappa_k) ,
& D_2(\eta)
& := \diag((\eta^\t \mean_1)^2, (\eta^\t \mean_2)^2, \dotsc,
(\eta^\t \mean_k)^2)
\end{align*}
and observe that
\begin{align*}
\grad^2f(\eta) & = \Means K D_2(\eta) \Means^\t .
\end{align*}
By assumption, the diagonal entries of $D_2(\phi) D_2(\psi)^{-1}$ are
distinct, and therefore
\begin{align*}
M_{\ICA}(\phi,\psi)
& = \bigl( \grad^2f(\phi) \bigr) \bigl( \grad^2f(\psi) \bigr)^{-1}
= \Means D_2(\phi) D_2(\psi)^{-1} \Means^{-1}
\end{align*}
is diagonalizable, and every eigenvalue has geometric multiplicity one.
\end{proof}

\section{Incoherence and random rotations}
\label{section:incoherence}

The multi-view technique of~\citet{AHK12} can be used to estimate mixtures
of product distributions, which include, as special cases, mixtures of
Gaussians with axis-aligned covariances $\Sig_i = \diag(\sigma_{1,i}^2,
\sigma_{2,i}^2, \dotsc, \sigma_{d,i}^2)$.
Spherical covariances $\Sig_i = \sigma_i^2 I$ are, of course, also
axis-aligned.
The idea is to randomly partition the coordinates $[d]$ into three groups,
inducing multiple ``views'' $x = (x_1,x_2,x_3)$ with each $x_t \in
\R^{d_t}$ for some $d_t \geq k$ such that $x_1$, $x_2$, and $x_3$ are
conditionally independent given $h$.
When the matrix of conditional means $A_t := [ \E[x_t | h = 1] | \E[x_t | h
= 2 ] | \dotsb | \E[x_t | h = k] ] \in \R^{d_t \times k}$ for each view $t
\in \{1,2,3\}$ has rank $k$, then an efficient technique similar to that
described in \refthm{spherical} will recover the parameters~\citep[for
details, see][]{AHK12,Tensor}.

\citet{AHK12} show that if $A$ has rank $k$ and also satisfies a mild
incoherence condition, then a random partitioning guarantees that each
$A_t$ has rank $k$, and lower-bounds the $k$-th largest singular value of
each $A_t$ by that of $A$.
The condition is similar to the spreading condition of \citet{CR08}.

Define $\coherence(A) := \max_{i \in [d]} \{ e_i^\t \Pi_A e_i \}$ to be the
largest diagonal entry of the ortho-projector $\Pi_A$ to the range of $A$.
When $A$ has rank $k$, we have $\coherence(A) \in [k/d,1]$; it is maximized
when $\range(A) = \Span\{e_1,e_2,\dotsc,e_k\}$ and minimized when the range
is spanned by a subset of the Hadamard basis of cardinality $k$.
Roughly speaking, if the matrix of conditional means has low coherence,
then its full-rank property is witnessed by many partitions of $[d]$; this
is made formal in the following lemma.
\begin{lemma} \label{lem:product}
Assume $A$ has rank $k$ and that $\coherence(A) \leq (\veps^2/6) /
\ln(3k/\delta)$ for some $\veps,\delta \in (0,1)$.
With probability at least $1-\delta$, a random partitioning of the
dimensions $[d]$ into three groups (for each $i \in [d]$, independently
pick $t \in \{1,2,3\}$ uniformly at random and put $i$ in group $t$)
has the following property.
For each $t \in \{1,2,3\}$, the matrix $A_t$ obtained by selecting the rows
of $A$ in group $t$ has full column rank, and the $k$-th largest singular
value of $A_t$ is at least $\sqrt{(1-\veps)/3}$ times that of $A$.
\end{lemma}

For a mixture of spherical Gaussians, one can randomly rotate $x$ before
applying the random coordinate partitioning.
This is because if $\Theta \in \R^{d \times d}$ is an orthogonal matrix,
then the distribution of $\tl{x} := \Theta x$ is also a mixture of
spherical Gaussians.
Its matrix of conditional means is given by $\tl{A} := \Theta A$.
The following lemma implies that multiplying a tall matrix $A$ by a random
rotation $\Theta$ causes the product to have low coherence.
\begin{lemma}[\citealp{ridge}] \label{lem:rotation}
Let $A \in \R^{d \times k}$ be a fixed matrix with rank $k$, and let
$\Theta \in \R^{d \times d}$ be chosen uniformly at random among all
orthogonal $d \times d$ matrices.
For any $\eta \in (0,1)$, with probability at least $1 - \eta$, the matrix
$\tl{A} := \Theta A$ satisfies
\[
\coherence(\tl{A}) \leq
\frac{k + \sqrt{2k \ln(d/\eta)} + 2\ln(d/\eta)}
{d \bigl( 1 - 1/(4d) - 1/(360d^3) \bigr)^2}
.
\]
\end{lemma}

Take $\eta$ from Lemma~\ref{lem:rotation} and $\veps,\delta$ from
Lemma~\ref{lem:product} to be constants.
Then the incoherence condition of Lemma~\ref{lem:product} is satsified
provided that $d \geq c \cdot (k \log k)$ for some positive constant $c$.

\section{Learning algorithm and finite sample analysis}
\label{section:finite}

In this section, we state and analyze a learning algorithm based on the
estimator from \refthm{spherical}, which assumed availability of exact
moments of $x$.
The proposed algorithm only uses a finite sample to estimate moments, and
also explicitly deals with the eigenvalue separation condition assumed in
\refthm{spherical} via internal randomization.

\subsection{Notation}

For a matrix $X \in \R^{m \times m}$, we use $\sig_t[X]$ to denote the
$t$-th largest singular value of a matrix $X$, and $\|X\|_2$ to denote its
spectral norm (so $\|X\|_2 = \sig_1[X]$).

For a third-order tensor $Y \in \R^{m \times m \times m}$ and $U,V,W \in
\R^{m \times n}$, we use the notation $Y[U,V,W] \in \R^{n \times n \times
n}$ to denote the third-order tensor given by
\begin{align*}
Y[U,V,W]_{j_1,j_2,j_3}
& = \sum_{1 \leq i_1,i_2,i_3 \leq m}
U_{i_1,j_1} V_{i_2,j_2} W_{i_3,j_3} Y_{i_1,i_2,i_3}
, \quad \forall j_1, j_2, j_3 \in [n] .
\end{align*}
Note that this is the analogue of $U^\t X V \in \R^{n \times n}$ for a
matrix $X \in \R^{m \times m}$ and $U, V \in \R^{m \times n}$.
For $Y \in \R^{m \times m \times m}$, we use $\|Y\|_2$ to denote its
operator (or supremum) norm $\|Y\|_2 := \sup \{ |Y[u,v,w]| : u,v,w \in
\R^m, \|u\|_2 = \|v\|_2 = \|w\|_2 = 1\}$.

\subsection{Algorithm}

The proposed algorithm, called \alg, is described in
\reffig{alg}.
The algorithm essentially implements the decomposition strategy in
\refthm{spherical} using plug-in moments.
To simplify the analysis, we split our sample (say, initially of size $2n$)
in two: we use the first half for empirical moments ($\h\mean$ and
$\wh\Moment_2$) used in constructing $\h\sigma^2$, $\wh{M}_2$, $\wh{W}$,
and $\wh{B}$; and we use the second half for empirical moments
($\wh{W}^\t\underline{\h\mean}$ and
$\underline{\wh\Moment_3}[\wh{W},\wh{W},\wh{W}]$ used in constructing
$\wh{M}_3[\wh{W},\wh{W},\wh{W}]$.
Observe that this ensures $\wh{M}_3$ is independent of $\wh{W}$.

Let $\{ (x_i,h_i) : i \in [n] \}$ be $n$ i.i.d.~copies of $(x,h)$, and
write $\dataset := \{x_1,x_2,\dotsc,x_n\}$.
Let $\underline\dataset$ be an independent copy of $\dataset$.
Furthermore, define the following moments and empirical moments:
\begin{align*}
\mean & := \E[x] ,
& \Moment_2 & := \E[xx^\t] ,
& \Moment_3 & := \E[x \otimes x \otimes x] , \\
\h\mean & := \frac1{|\dataset|} \sum_{x \in \dataset} x ,
& \wh\Moment_2 & := \frac1{|\dataset|} \sum_{x \in \dataset} xx^\t ,
& \underline{\wh\Moment_3} & := \frac1{|\underline\dataset|} \sum_{x \in
\underline\dataset} x \otimes x \otimes x ,
& \underline{\h\mean} & := \frac1{|\underline\dataset|} \sum_{x \in
\underline\dataset} x .
\end{align*}
So $\dataset$ represents the first half of the sample, and
$\underline\dataset$ represents the second half of the sample.

\begin{figure}
\framebox[\textwidth]{\begin{minipage}{0.95\textwidth}
\alg
\begin{enumerate}
\item Using the first half of the sample, compute empirical mean $\h\mean$
and empirical second-order moments $\wh\Moment_2$.

\item Let $\h\sigma^2$ be the $k$-th largest eigenvalue of the empirical
covariance matrix $\wh\Moment_2 - \h\mean\h\mean^\t$.

\item Let $\wh{M}_2$ be the best rank-$k$ approximation to $\wh\Moment_2 -
\h\sigma^2 I$
\[ \wh{M}_2 := \arg\min_{X \in \R^{d \times d} : \rank(X) \leq k}
\|(\wh\Moment_2 - \h\sigma^2 I) - X\|_2 \]
which can be obtained via the singular value decomposition.

\item Let $\wh{U} \in \R^{d \times k}$ be the matrix of left orthonormal
singular vectors of $\wh{M}_2$.

\item Let $\wh{W} := \wh{U} (\wh{U}^\t \wh{M}_2 \wh{U})^{\dag1/2}$, where
$X^\dag$ denotes the Moore-Penrose pseudoinverse of a matrix $X$.

Also define $\wh{B} := \wh{U} (\wh{U}^\t \wh{M}_2 \wh{U})^{1/2}$.

\item Using the second half of the sample, compute whitened empirical
averages $\wh{W}^\t \underline{\h\mean}$ and third-order moments
$\underline{\wh\Moment_3}[\wh{W},\wh{W},\wh{W}]$.

\item Let $\wh{M}_3[\wh{W},\wh{W},\wh{W}]
:= \underline{\wh\Moment_3}[\wh{W},\wh{W},\wh{W}]
- \h\sigma^2 \sum_{i=1}^d \bigl(
(\wh{W}^\t \underline{\h\mean}) \otimes (\wh{W}^\t e_i) \otimes (\wh{W}^\t
e_i) + (\wh{W}^\t e_i) \otimes (\wh{W}^\t \underline{\h\mean}) \otimes
(\wh{W}^\t e_i) + (\wh{W}^\t e_i) \otimes (\wh{W}^\t e_i) \otimes
(\wh{W}^\t \underline{\h\mean}) \bigr)$.

\item Repeat the following steps $t$ times (where $t :=
\lceil\log_2(1/\delta)\rceil$
for confidence $1-\delta$):
\begin{enumerate}
\item Choose $\theta \in \R^k$ uniformly at random from the unit sphere in
$\R^k$.

\item Let $\{ (\h{v}_i,\h\lambda_i) : i \in [k] \}$ be the
eigenvector/eigenvalue pairs of $\wh{M}_3[\wh{W},\wh{W},\wh{W}\theta]$.

\end{enumerate}
Retain the results for which $\min \bigl( \{ |\h\lambda_i - \h\lambda_j| :
i \neq j \} \cup \{ |\h\lambda_i| : i \in [k] \} \bigr)$ is largest.

\item Return the parameter estimates $\h\sigma^2$,
\begin{align*}
\h\mean_i & := \frac{\h\lambda_i}{\theta^\t\h{v}_i}
\wh{B} \h{v}_i , \quad i \in [k] , \\
\h{w} & := [ \h\mean_1 | \h\mean_2 | \dotsb | \h\mean_k ]^{\dag} \h\mean .
\end{align*}

\end{enumerate}
\end{minipage}}
\caption{Algorithm for learning mixtures of Gaussians with common spherical
covariance.}
\label{fig:alg}
\end{figure}

\subsection{Structure of the moments}

We first recall the basic structure of the moments $\mean$, $\Moment_2$,
and $\Moment_3$ as established in \refthm{spherical}; for simplicity, we
restrict to the special case where $\sigma_1^2 = \sigma_2^2 = \dotsb =
\sigma_k^2 = \sigma^2$.

\begin{lemma}[Structure of moments] \label{lem:structure}
\begin{align*}
\mean & = \sum_{i=1}^k w_i \mean_i , \\
\Moment_2 & = \sum_{i=1}^k w_i \mean_i \mean_i^\t  + \sigma^2 I , \\
\Moment_3 & = \sum_{i=1}^k w_i \mean_i \otimes \mean_i \otimes \mean_i
+ \sigma^2 \sum_{j=1}^d
\Bigl( \mean \otimes e_j \otimes e_j
+ e_j \otimes \mean \otimes e_j
+ e_j \otimes e_j \otimes \mean \Bigr)
.
\end{align*}
\end{lemma}

\subsection{Concentration behavior of empirical quantities}

In this subsection, we prove concentration properties of empirical
quantities based on $\dataset$; clearly the same properties hold for
$\underline\dataset$.

Let $\dataset_i := \{ x_j \in \dataset : h_j = i \}$ and $\h{w}_i :=
|\dataset_i|/|\dataset|$ for $i \in [k]$.
Also, define the following (empirical) conditional moments:
\begin{align*}
\mean_i & := \E[x | h = i] ,
& \Moment_{2,i} & := \E[xx^\t | h = i] ,
& \Moment_{3,i} & := \E[x \otimes x \otimes x | h = i] , \\
\h\mean_i & := \frac1{|\dataset_i|} \sum_{x \in \dataset_i} x ,
& \wh\Moment_{2,i} & := \frac1{|\dataset_i|} \sum_{x \in \dataset_i} xx^\t
,
& \wh\Moment_{3,i} & := \frac1{|\dataset_i|} \sum_{x \in \dataset_i} x
\otimes x \otimes x .
\end{align*}

\begin{lemma}[Concentration of proportions] \label{lem:w-conc}
Pick any $\delta \in (0,1/2)$.
With probability at least $1-2\delta$,
\begin{align*}
| \h{w}_i - w_i |
& \leq
\sqrt{\frac{2w_i(1-w_i)\ln(2k/\delta)}{n}} + \frac{2\ln(2k/\delta)}{3n}
, \quad \forall i \in [k] ; \\
\biggl( \sum_{i=1}^k (\h{w}_i - w_i)^2 \biggr)^{1/2}
& \leq \frac{1 + \sqrt{\ln(1/\delta)}}{\sqrt{n}} .
\end{align*}
\end{lemma}
\begin{proof}
The first inequality follows from Bernstein's inequality and a union bound.
The second inequality follows from a simple application of McDiarmid's
inequality \citep[see][Proposition 19]{HKZ12}.
\end{proof}

\begin{lemma}[Concentration of per-component empirical moments]
\label{lem:conditional-moments-conc}
Pick any $\delta \in (0,1)$ and any matrix $R \in \R^{d \times r}$ of rank
$r$.
\begin{enumerate}
\item First-order moments: with probability at least $1-\delta$,
\begin{equation*}
\|R^\t(\h\mean_i - \mean_i)\|_2
\leq \sigma \|R\|_2 \sqrt{\frac{r + 2\sqrt{r\ln(k/\delta)} +
2\ln(k/\delta)}{\h{w}_i n}}
, \quad \forall i \in [k] .
\end{equation*}
\item Second-order moments: with probability at least $1-\delta$,
\begin{multline*}
\|R^\t(\wh\Moment_{2,i} - \Moment_{2,i})R\|_2
\leq
\sigma^2 \|R\|_2^2
\Biggl( \sqrt{\frac{128 (r \ln9 + \ln(2k/\delta))}{\h{w}_i n}}
+ \frac{4 (r \ln9 + \ln(2k/\delta))}{\h{w}_i n} \Biggr)
\\
+ 2\sigma \|R^\t\mean_i\|_2 \|R\|_2
\sqrt{\frac{r + 2\sqrt{r\ln(2k/\delta)} + 2\ln(2k/\delta)}{\h{w}_i n}}
, \quad \forall i \in [k] .
\end{multline*}
\item Third-order moments: with probability at least $1-\delta$,
\begin{multline*}
\|(\wh\Moment_{3,i} - \Moment_{3,i})[R,R,R]\|_2
\leq
\sigma^3 \|R\|_2^3
\sqrt{\frac{108e^3 \lceil r\ln13 + \ln(3k/\delta)\rceil^3}{\h{w}_in}}
\\
+ 3\sigma^2 \|R^\t\mean_i\|_2 \|R\|_2^2
\Biggl( \sqrt{\frac{128 (r \ln9 + \ln(3k/\delta))}{\h{w}_i n}}
+ \frac{4 (r \ln9 + \ln(3k/\delta))}{\h{w}_i n} \Biggr)
\\
+ 3\sigma\|R^\t\mean_i\|_2^2 \|R\|_2
\sqrt{\frac{r + 2\sqrt{r\ln(3k/\delta)} + 2\ln(3k/\delta)}{\h{w}_i n}}
, \quad \forall i \in [k] .
\end{multline*}
\end{enumerate}
\end{lemma}
\begin{proof}
We separately consider first-, second-, and third-order moments.
Throughout, we let the thin SVD of $R$ be given by $R = USV^\t$, where $U
\in \R^{d \times r}$ has orthonormal columns, and $\|VS\|_2 = \|R\|_2$.

\empar{First-order moments.}
Observe that $(\h{w}_in/\sigma^2) \|U^\t(\h\mean_i - \mean_i)\|_2^2$ is
distributed as the sum of $r$ independent $\chi^2$ random variables, each
with one degree of freedom.
Thus, \reflem{chi2} and union bounds imply
\begin{gather*}
\Pr\Biggl[ \exists i \in [k] \centerdot \|U^\t(\h\mean_i - \mean_i)\|_2^2
> \sigma^2 \biggl( \frac{r + 2\sqrt{r\ln(k/\delta)} +
2\ln(k/\delta)}{\h{w}_i n} \biggr)
\Biggr]
\leq \delta .
\end{gather*}

\empar{Second-order moments.}
Since $\Moment_{2,i} = \sigma^2 I + \mean_i \mean_i^\t$, it follows by the
triangle and Cauchy-Schwarz inequalities that
\begin{align*}
\|R^\t(\wh\Moment_{2,i} - \Moment_{2,i})R\|_2
& \leq \|R\|_2^2 \biggl\|
\frac1{\h{w}_in} \sum_{j \in [n] : x_j \in \dataset_i}
\Bigl( U^\t(x_j - \mean_i) (x_j - \mean_i)^\t U - \sigma^2 I \Bigr)
\biggr\|_2
\\
& \qquad
+ 2\|R^\t\mean_i\|_2 \|R^\t(\h\mean_i - \mean_i)\|_2 .
\end{align*}
A tail bound for the first term follows from \reflem{normal-matrix},
combined with a union bound:
\begin{multline*}
\Pr\Biggl[ \exists i \in [k] \centerdot
\biggl\|
\frac1{\h{w}_in} \sum_{j \in [n] : x_j \in \dataset_i}
\Bigl( U^\t(x_j - \mean_i) (x_j - \mean_i)^\t U - \sigma^2 I \Bigr)
\biggr\|_2
\\
>
\sigma^2 \Biggl( \sqrt{\frac{128 (r \ln9 + \ln(k/\delta))}{\h{w}_i n}}
+ \frac{4 (r \ln9 + \ln(k/\delta))}{\h{w}_i n}
\Biggr)
\Biggr]
\leq \delta
.
\end{multline*}
The second term is handled as above.

\empar{Third-order moments.}
It can be checked that
\begin{align*}
\Moment_{3,i}
& = \mean_i \otimes \mean_i \otimes \mean_i
+ \sigma^2 \sum_{\iota=1}^d \biggl( \mean_i \otimes e_{\iota} \otimes
e_{\iota}
+ e_{\iota} \otimes \mean_i \otimes e_{\iota}
+ e_{\iota} \otimes e_{\iota} \otimes \mean_i \biggr)
\end{align*}
(similar to \reflem{structure})
and
\begin{align*}
\lefteqn{
\wh\Moment_{3,i} - \Moment_{3,i}
} \\
& =
\frac{1}{\h{w}_in} \Biggl( \sum_{j \in [n] : x_j \in \dataset_i}
(x_j - \mean_i) \otimes (x_j - \mean_i) \otimes (x_j -
\mean_i)
\\
& \qquad
+ \sum_{j \in [n] : x_j \in \dataset_i}
\Bigl(
\mean_i \otimes (x_j - \mean_i) \otimes (x_j - \mean_i)
- \sigma^2 \sum_{\iota=1}^d \mean_i \otimes e_{\iota} \otimes e_{\iota}
\Bigr)
\\
& \qquad
+ \sum_{j \in [n] : x_j \in \dataset_i}
\Bigl(
(x_j - \mean_i) \otimes \mean_i \otimes (x_j - \mean_i)
- \sigma^2 \sum_{\iota=1}^d e_{\iota} \otimes \mean_i \otimes e_{\iota}
\Bigr)
\\
& \qquad
+ \sum_{j \in [n] : x_j \in \dataset_i}
\Bigl(
(x_j - \mean_i) \otimes (x_j - \mean_i) \otimes \mean_i
- \sigma^2 \sum_{\iota=1}^d e_{\iota} \otimes e_{\iota} \otimes \mean_i
\Bigr)
\\
& \qquad
+ \sum_{j \in [n] : x_j \in \dataset_i}
\mean_i \otimes \mean_i \otimes (x_j - \mean_i)
+ \sum_{j \in [n] : x_j \in \dataset_i}
\mean_i \otimes (x_j - \mean_i) \otimes \mean_i
+ \sum_{j \in [n] : x_j \in \dataset_i}
(x_j - \mean_i) \otimes \mean_i \otimes \mean_i
\Biggr)
.
\end{align*}
Therefore, by the triangle and Cauchy-Schwarz inequalities,
\begin{align*}
\|(\wh\Moment_{3,i} - \Moment_{3,i})[R,R,R]\|_2
& \leq \|R\|_2^3
\biggl\| \frac{1}{\h{w}_in} \sum_{j \in [n] : x_j \in \dataset_i}
U^\t(x_j - \mean_i) \otimes 
U^\t(x_j - \mean_i) \otimes 
U^\t(x_j - \mean_i)
\biggr\|_2
\\
& \quad
+ 3 \|R^\t\mean_i\|_2 \biggl\|
\frac{1}{\h{w}_in} \sum_{j \in [n] : x_j \in \dataset_i}
R^\t
\Bigl( (x_j - \mean_i) (x_j - \mean_i)^\t - \sigma^2 I \Bigr)
R
\biggr\|_2
\\
& \quad
+ 3 \|R^\t\mean_i\|_2^2 \|R^\t(\h\mean_i - \mean_i)\|_2 .
\end{align*}
A tail bound for the first term is given by \reflem{normal-tensor},
combined with a union bound:
\begin{multline*}
\Pr\Biggl[
\exists i \in [k] \centerdot
\biggl\| \frac1{\h{w}_in} \sum_{j \in [n] : x_j \in \dataset_i}
\frac1{\sigma^3}
U^\t(x_j - \mean_i)
\otimes U^\t(x_j - \mean_i)
\otimes U^\t(x_j - \mean_i)
\biggr\|_2 \\
> \sigma^3
\sqrt{\frac{108e^3 \lceil r\ln13 + \ln(k/\delta)\rceil^3}{\h{w}_in}}
\Biggr] \leq \delta .
\end{multline*}
The other terms are handled as per above.
\end{proof}

\begin{lemma}[Accuracy of empirical moments] \label{lem:moments}
Fix a matrix $R \in \R^{d \times r}$.
Define
$\Bound_{1,R} := \max_{i \in [k]} \|R^\t\mean_i\|_2$,
$\Bound_{2,R} := \max_{i \in [k]} \|R^\t\Moment_{2,i}R\|_2$,
$\Bound_{3,R} := \max_{i \in [k]} \|\Moment_{3,i}[R,R,R]\|_2$,
$\Err_{1,R} := \max_{i \in [k]} \|R^\t(\h\mean_i - \mean_i)\|_2$,
$\Err_{2,R} := \max_{i \in [k]} \|R^\t(\wh\Moment_{2,i} -
\Moment_{2,i})R\|_2$,
$\Err_{3,R} := \max_{i \in [k]} \|\wh\Moment_{3,i}[R,R,R] -
\Moment_{3,i}[R,R,R]\|_2$, and
$\Err_w := (\sum_{i=1}^k (\h{w}_i - w_i)^2)^{1/2}$.
Then
\begin{align*}
\|R^\t(\h\mean - \mean)\|_2
& \leq (1 + \sqrt{k}\Err_w) \Err_{1,R} + \sqrt{k} \Bound_{1,R} \Err_w ; \\
\|R^\t(\wh\Moment_2 - \Moment_2)R\|_2
& \leq (1 + \sqrt{k}\Err_w) \Err_{2,R} + \sqrt{k} \Bound_{2,R} \Err_w ; \\
\|(\wh\Moment_3 - \Moment_3)[R,R,R]\|_2
& \leq (1 + \sqrt{k}\Err_w) \Err_{3,R} + \sqrt{k} \Bound_{3,R} \Err_w .
\end{align*}
\end{lemma}
\begin{proof}
We just show the third claimed inequalitiy, as the others are similar.
Write as shorthand $\wh{T}_i := \wh\Moment_{3,i}[R,R,R]$ and
$T_i := \Moment_{3,i}[R,R,R]$.
Then
\begin{align*}
\biggl\| \sum_{i=1}^k \h{w}_i \wh{T}_i - \sum_{i=1}^k w_i T_i \biggr\|_2
& \leq 
\biggl\|\sum_{i=1}^k w_i (\wh{T}_i - T_i) \biggr\|_2
+ \biggl\|\sum_{i=1}^k (\h{w}_i - w_i) T_i \biggr\|_2
+ \biggl\|\sum_{i=1}^k (\h{w}_i - w_i) (\wh{T}_i - T_i)
\biggr\|_2
\\
& \leq
\sum_{i=1}^k w_i \|\wh{T}_i - T_i\|_2
+ \sum_{i=1}^k |\h{w}_i - w_i| \|T_i\|_2
+ \sum_{i=1}^k |\h{w}_i - w_i| \|\wh{T}_i - T_i\|_2
\\
& \leq
\max_{i \in [k]} \|\wh{T}_i - T_i\|_2
+ \sqrt{k} \|\h{w} - w\|_2 \max_{i \in [k]} \|T_i\|_2
+ \sqrt{k} \|\h{w} - w\|_2 \max_{i \in [k]} \|\wh{T}_i - T_i\|_2
\\
& = \Err_{3,R} + \sqrt{k} \Err_w \Bound_{3,R} + \sqrt{k} \Err_w \Err_{3,R}
\end{align*}
where the first and second steps use the triangle inequality,
and the second step uses H\"older's inequality.
\end{proof}

\subsection{Estimation of $\sigma^2$, $M_2$, and $M_3$}

The covariance matrix can be written as $\Moment_2 - \mean\mean^\t$, and
the empirical covariance matrix can be written as $\wh\Moment_2 -
\h\mean\h\mean^\t$.
Recall that the estimate of $\sigma^2$, denoted by $\h\sigma^2$, is given
by the $k$-th largest eigenvalue of the empirical covariance matrix
$\wh\Moment_2 - \h\mean\h\mean^\t$; and that the estimate of $M_2$, denoted
by $\wh{M}_2$, is the best rank-$k$ approximation to $\wh\Moment_2 -
\h\sigma^2 I$.
Of course, the singular values of a positive semi-definite matrix are the
same as its eigenvalues; in particular, $\h\sigma^2 = \sig_k[\wh\Moment_2
- \h\mean\h\mean^\t]$.

\begin{lemma}[Accuracy of $\h\sigma^2$ and $\wh{M}_2$]
\label{lem:M2}
~
\begin{enumerate}
\item $|\h\sigma^2 - \sigma^2|
\leq \|\wh\Moment_2 - \Moment_2\|_2 + 2\|\mean\|_2\|\h\mean - \mean\|_2 +
\|\h\mean - \mean\|_2^2$.

\item $\|\wh{M}_2 - M_2\|_2 \leq
4\|\wh\Moment_2 - \Moment_2\|_2 + 4\|\mean\|_2\|\h\mean - \mean\|_2 +
2\|\h\mean - \mean\|_2^2$.
\end{enumerate}
\end{lemma}
\begin{proof}
Using Weyl's inequality~\citep[Theorem 4.11, p.~204]{SS90}, we obtain
$|\sig_k[\wh\Moment_2 - \h\mean\h\mean^\t] - \sig_k[\Moment_2 -
\mean\mean^\t]| \leq \|(\wh\Moment_2 - \h\mean\h\mean^\t) - (\Moment_2 -
\mean\mean^\t)\|_2 \leq \|\wh\Moment_2 - \Moment_2\|_2 +
2\|\mean\|_2\|\h\mean - \mean\|_2 + \|\h\mean - \mean\|_2^2$.
The first claim then follows by observing that $\sig_k[\Moment_2 -
\mean\mean^\t] = \sigma^2$ as per \refthm{moments}.

For the second claim, observe that
$\sig_{k+1}(\Moment_2 - \sigma^2 I) = 0$ as $\Moment_2 - \sigma^2 I$ has
rank $k$.
Therefore $\sig_{k+1}(\wh\Moment_2 - \h\sigma^2 I) =
|\sig_{k+1}(\wh\Moment_2 - \h\sigma^2 I) - \sig_{k+1}(\Moment_2 -
\sigma^2 I)| \leq \|(\wh\Moment_2 -\h\sigma^2 I) - (\Moment_2 - \sigma^2
I)\|_2$, again using Weyl's inequality.
Since $\wh{M}_2$ is the best rank-$k$ approximation to $\wh\Moment_2 -
\h\sigma^2 I$, it follows that $\|\wh{M}_2 - (\wh\Moment_2 - \h\sigma^2
I)\|_2 \leq \sig_{k+1}(\wh\Moment_2 - \h\sigma^2 I) \leq \|(\wh\Moment_2
-\h\sigma^2 I) - (\Moment_2 - \sigma^2 I)\|_2$.
Therefore
$\|\wh{M}_2 - M_2\|_2 \leq \|(\wh\Moment_2 - \h\sigma^2 I) - (\Moment_2 -
\sigma^2 I)\|_2 + \|\wh{M}_2 - (\wh\Moment_2 - \h\sigma^2 I)\|_2 \leq
2\|(\wh\Moment_2 - \h\sigma^2 I) - (\Moment_2 - \sigma^2 I)\|_2 \leq
2\|\wh\Moment_2 - \Moment_2\|_2 + 2|\h\sigma^2 - \sigma^2| \leq
4\|\wh\Moment_2 - \Moment_2\|_2 + 4\|\mean\|_2\|\h\mean - \mean\|_2 +
2\|\h\mean - \mean\|_2^2$.
\end{proof}

Recall that the estimate of $M_3$, denoted by $\wh{M}_3$, is given by
$\underline{\wh\Moment_3} - \h\sigma^2 \sum_{i=1}^d (\underline{\h\mean}
\otimes e_i \otimes e_i + e_i \otimes \underline{\h\mean} \otimes e_i + e_i
\otimes e_i \otimes \underline{\h\mean})$.

\begin{lemma}[Accuracy of $\wh{M}_3$]
\label{lem:M3}
For any matrix $R \in \R^{d \times r}$,
\begin{align*}
\|\wh{M}_3[R,R,R] - M_3[R,R,R]\|_2
& \leq \|\underline{\wh\Moment_3}[R,R,R] - \Moment_3[R,R,R]\|_2
\\
& \qquad
+ 3 \|R\|_2^2 (\|R^\t(\underline{\h\mean} - \mean)\|_2 + \|R^\t\mean\|_2) \\
& \qquad\qquad
(\|\wh\Moment_2 - \Moment_2\|_2 + 2\|\mean\|_2\|\underline{\h\mean} -
\mean\|_2 + \|\underline{\h\mean} - \mean\|_2^2) \\
& \qquad
+ \sigma^2 \|R\|_2^2 \|R^\t(\underline{\h\mean} - \mean)\|_2 .
\end{align*}
\end{lemma}
\begin{proof}
Let $\wh{G} := \sum_{i=1}^d (\underline{\h\mean} \otimes e_i \otimes e_i +
e_i \otimes \underline{\h\mean} \otimes e_i + e_i \otimes e_i \otimes
\underline{\h\mean} $ and $G := \sum_{i=1}^d (\mean \otimes e_i \otimes e_i
+ e_i \otimes \mean \otimes e_i + e_i \otimes e_i \otimes \mean)$.
Then
\begin{align*}
\lefteqn{
\|(\wh{M}_3 - M_3)[R,R,R]\|_2
} \\
& = \|(\underline{\wh\Moment_3} - \Moment_3)[R,R,R]
- (\h\sigma^2 - \sigma^2) (\wh{G} - G)[R,R,R]
- (\h\sigma^2 - \sigma^2) G[R,R,R]
- \sigma^2 (\wh{G} - G)[R,R,R]\|_2 \\
& \leq \|(\underline{\wh\Moment_3} - \Moment_3)[R,R,R]\|_2
+ |\h\sigma^2 - \sigma^2| \|(\wh{G} - G)[R,R,R]\|_2
+ |\h\sigma^2 - \sigma^2| \|G[R,R,R]\|_2
\\
& \qquad
+ \sigma^2 \|(\wh{G} - G)[R,R,R]\|_2 .
\end{align*}
Observe that by the triangle inequality, $\|G[R,R,R]\|_2 \leq 3\|R\|_2^2
\|R^\t \mean\|_2$ and $\|(\wh{G} - G)[R,R,R]\|_2 \leq 3\|R\|_2^2
\|R^\t(\underline{\h\mean} - \mean)\|_2$.
Furthermore, by \reflem{M2}, we have $|\h\sigma^2 - \sigma^2| \leq
\|\wh\Moment_2 - \Moment_2\|_2 + 2\|\mean\|_2\|\underline{\h\mean} -
\mean\|_2 + \|\underline{\h\mean} - \mean\|_2^2$.
Therefore the claim follows.
\end{proof}

\subsection{Properties of projection and whitening operators}

Recall that $\wh{U} \in \R^{d \times k}$ is the matrix of left orthonormal
singular vectors of $\wh{M}_2$, and let $\wh{S} \in \R^{k \times k}$ be the
diagonal matrix of corresponding singular values.
Analogously define $U$ and $S$ relative to $M_2$.

Define $\Err_{M_2} := \|\wh{M}_2 - M_2\|_2 / \sig_k[M_2]$.

\begin{lemma}[Properties of projection operators] \label{lem:svd}
Assume $\Err_{M_2} \leq 1/3$.
Then
\begin{enumerate}
\item $(1 + \Err_{M_2}) S \succeq \wh{U}^\t \wh{M}_2 \wh{U} = \wh{S}
\succeq (1 - \Err_{M_2}) S \succ 0$.

\item $\sig_k[\wh{U}^\t U] \geq \sqrt{1 - (9/4)\Err_{M_2}^2} > 0$.

\item $\sig_k[\wh{U}^\t M_2 \wh{U}] \geq (1-(9/4)\Err_{M_2}^2)
\sig_k[M_2] > 0$.

\item $\|(I - \wh{U}\wh{U}^\t)UU^\t\|_2 \leq (3/2)\Err_{M_2}$.

\end{enumerate}
\end{lemma}
\begin{proof}
By the assumptions that $\Err_{M_2} \leq 1/3$ and $M_2$ is symmetric
positive definite, we have
\begin{equation*}
|\sig_t[\wh{M}_2] - \sig_t[M_2]| \leq
\|\wh{M}_2 - M_2\|_2 \leq \Err_{M_2} \sig_k[M_2]
, \quad \forall t \in [k]
\end{equation*}
by Weyl's inequality.
Therefore $\wh{M}_2$ is symmetric positive definite, and
\begin{equation*}
(1 + \Err_{M_2}) S \succeq
\wh{U}^\t \wh{M}_2 \wh{U} = \wh{S}
\succeq (1 - \Err_{M_2}) S ,
\end{equation*}
which proves the first claim.

Now let $\wh{U}_\perp \in \R^{d \times (d-k)}$ be a matrix with orthonormal
columns spanning the orthogonal complement of the range of $\wh{M}_2$.
By Wedin's theorem~\citep[Theorem 4.4, p.~262]{SS90} and the assumption
that $\Err_{M_2} \leq 1/3$,
\begin{equation*}
\|\wh{U}_\perp^\t U\|_2
\leq \frac{\|\wh{M}_2 - M_2\|_2}{\sig_k[\wh{M}_2]}
\leq \frac{\Err_{M_2}}{1-\Err_{M_2}} \leq \frac{3}{2}\Err_{M_2} \leq
\frac12 .
\end{equation*}
Therefore $\wh{U}^\t U$ is non-singular, and for any $v \in \R^k$,
$\|\wh{U}^\t Uv\|_2^2 = 1 - \|\wh{U}_\perp^\t Uv\|_2^2 \geq 1 -
(9/4) \Err_{M_2}^2$, which in turn implies the second claim.
The second claim then implies that $\sig_k[\wh{U}^\t M_2 \wh{U}] \geq
\sig_k[\wh{U}^\t U]^2 \sig_k[M_2] \geq (1-(9/4)\Err_{M_2}^2)
\sig_k[M_2]$, which gives the third claim.
For the final claim, observe that
$\|(I - \wh{U}\wh{U}^\t)UU^\t\|_2 = \|\wh{U}_\perp \wh{U}_\perp^\t
UU^\t\|_2 = \|\wh{U}_\perp^\t U\|_2 \leq (3/2)\Err_{M_2}$, using the fact
that $\wh{U}_\perp$ and $U$ have orthonormal columns, and the above
displayed inequality.
\end{proof}

Recall that $\wh{W} = \wh{U} (\wh{U}^\t \wh{M}_2 \wh{U})^{\dag1/2}$.

\begin{lemma}[Properties of whitening operators] \label{lem:whitening}
Define $W := \wh{W} (\wh{W}^\t M_2 \wh{W})^{\dag1/2}$.
Assume $\Err_{M_2} \leq 1/3$.
Then
\begin{enumerate}
\item $\wh{W}^\t M_2 \wh{W}$ is symmetric positive definite, $W^\t
M_2 W = I$, and $W^\t \Means\diag(w)^{1/2}$ is orthogonal.

\item $\|\wh{W}\|_2 \leq \frac1{\sqrt{(1-\Err_{M_2}) \sig_k[M_2]}}$.

\item $\|(\wh{W}^\t M_2 \wh{W})^{1/2} - I\|_2 \leq (3/2)\Err_{M_2}$, \\
$\|(\wh{W}^\t M_2 \wh{W})^{-1/2} - I\|_2 \leq (3/2)\Err_{M_2}$, \\
$\|\wh{W}^\t \Means\diag(w)^{1/2}\|_2 \leq \sqrt{1 + (3/2)\Err_{M_2}}$, \\
$\|(\wh{W} - W)^\t \Means\diag(w)^{1/2}\|_2 \leq (3/2)\sqrt{1 +
(3/2)\Err_{M_2}} \Err_{M_2}$.

\end{enumerate}
\end{lemma}
\begin{proof}
By \reflem{svd} (first and third claims), the matrices $\wh{U}^\t \wh{M}_2
\wh{U}$ and $\wh{U}^\t M_2 \wh{U}$ are symmetric positive definite.
Therefore
\begin{align*}
\wh{W}^\t M_2 \wh{W}
& =
(\wh{U}^\t \wh{M}_2 \wh{U})^{-1/2}
(\wh{U}^\t M_2 \wh{U})
(\wh{U}^\t \wh{M}_2 \wh{U})^{-1/2}
\succ 0
,
\\
W
& = \wh{W} (\wh{W}^\t M_2 \wh{W})^{-1/2}
,
\\
W^\t M_2 W
& =
(\wh{W}^\t M_2 \wh{W})^{-1/2}
(\wh{W}^\t M_2 \wh{W})
(\wh{W}^\t M_2 \wh{W})^{-1/2}
= I
.
\end{align*}
Since $M_2 = \Means\diag(w)\Means^\t$, it follows from the third equation
above that $W^\t \Means\diag(w)^{1/2}$ is orthogonal.
Thus the first claim is established.

For the second claim, note that
\begin{align*}
\|\wh{W}\|_2
& \leq \|(\wh{U}^\t \wh{M}_2 \wh{U})^{-1/2}\|_2
= \sig_k[\wh{U}^\t \wh{M}_2 \wh{U}]^{-1/2}
\leq ((1-\Err_{M_2})\sig_k[M_2])^{-1/2}
\end{align*}
where the last inequality follows from \reflem{svd} (first claim).

To show the third claim, we first bound $\|\wh{W}^\t M_2 \wh{W} - I\|_2$ as
\begin{align*}
\|\wh{W}^\t M_2 \wh{W} - I\|_2
& = \|\wh{W}^\t (M_2 - \wh{M}_2) \wh{W}\|_2 \\
& \leq \|\wh{W}\|_2^2 \|M_2 - \wh{M}_2\|_2 \\
& \leq \frac{\Err_{M_2}}{1-\Err_{M_2}} \leq \frac32 \Err_{M_2} \leq \frac12
\end{align*}
where the second inequality follows from the first claim.
This implies that every eigenvalue of $\wh{W}^\t M_2 \wh{W}$ is contained
in the interval of radius $(3/2)\Err_{M_2}$ around $1$.
Because $|(1+x)^{-1/2}-1| \leq |x|$ for all $|x| \leq 1/2$, the same is
true of the eigenvalues of $(\wh{W}^\t M_2 \wh{W})^{-1/2}$:
\begin{align*}
\|(\wh{W}^\t M_2 \wh{W})^{-1/2} - I\|_2
& \leq \frac32 \Err_{M_2}
.
\end{align*}
Furthermore,
\begin{align*}
\|\wh{W}^\t \Means\diag(w)^{1/2}\|_2^2
& = \|\wh{W}^\t M_2 \wh{W}\|_2 \\
& = \|I + \wh{W}^\t M_2 \wh{W} - I\|_2 \\
& \leq 1 + \|\wh{W}^\t M_2 \wh{W} - I\|_2 \leq 1 + \frac32 \Err_{M_2}
,
\end{align*}
so
\begin{align*}
\|(\wh{W} - W)^\t \Means\diag(w)^{1/2}\|_2
& = \|(I - (\wh{W}^\t M_2 \wh{W})^{-1/2}) \wh{W}^\t
\Means\diag(w)^{1/2}\|_2 \\
& \leq \|I - (\wh{W}^\t M_2 \wh{W})^{-1/2}\|_2 \|\wh{W}^\t
\Means\diag(w)^{1/2}\|_2
\leq \frac32 \Err_{M_2} \sqrt{1 + \frac32 \Err_{M_2}}
.
\end{align*}
This establishes the third claim.
\end{proof}

Define $\wh{T} := \wh{M}_3[\wh{W},\wh{W},\wh{W}]$ and $T := M_3[W,W,W]$,
both symmetric tensors in $\R^{k \times k \times k}$.
Also, define $\wh{T}[u] := \wh{M}_3[\wh{W},\wh{W},\wh{W}u]$ and $T[u] :=
M_3[W,W,Wu]$, both symmetric matrices in $\R^{k \times k}$.

\begin{lemma}[Tensor structure] \label{lem:tensor-structure}
The tensor $T$ can be written as
\begin{align*}
T
& = \sum_{i=1}^k \frac1{\sqrt{w_i}}
(W^\t \Means\diag(w)^{1/2} e_i)
\otimes (W^\t \Means\diag(w)^{1/2} e_i)
\otimes (W^\t \Means\diag(w)^{1/2} e_i)
\end{align*}
where the vectors $\{ W^\t \Means\diag(w)^{1/2} e_i : i \in [k] \}$ are
orthonormal.
Furthermore, the eigenvectors of $T[u]$ are
$\{ W^\t \Means\diag(w)^{1/2} e_i : i \in [k] \}$
and the corresponding eigenvalues are $\{ u^\t W^\t \mean_i : i \in [k]
\}$.
\end{lemma}
\begin{proof}
The structure of $T$ follows from \reflem{structure}, and the orthogonality
of $\{ W^\t \Means\diag(w)^{1/2} e_i : i \in [k] \}$ follows from
\reflem{whitening} (first claim).
The eigendecomposition of $T[u]$ is then readily seen from the structure of
$T$.
\end{proof}

\begin{lemma}[Tensor accuracy] \label{lem:tensor}
Assume $\Err_{M_2} \leq 1/3$.
Then
\begin{align*}
\|\wh{T} - T\|_2
& \leq \|\wh{M}_3[\wh{W},\wh{W},\wh{W}] - M_3[\wh{W},\wh{W},\wh{W}]\|_2
+ \frac{6}{\sqrt{w_{\min}}} \Err_{M_2}
.
\end{align*}
\end{lemma}
\begin{proof}
By \reflem{tensor-structure}, $T = \sum_{i=1}^k w_i^{-1/2} v_i \otimes v_i
\otimes v_i$ for some orthonormal vectors $\{ v_i : i \in [k] \}$, so
$\|T\|_2 \leq 1/\sqrt{w_{\min}}$.
By \reflem{whitening} (first and third claims),
$\wh{W} = W (\wh{W}^\t M_2 \wh{W})^{1/2}$
and
$\|(\wh{W}^\t M_2 \wh{W})^{1/2} - I\|_2 \leq (3/2) \Err_{M_2}$.
Therefore
\begin{align*}
\lefteqn{
\|M_3[\wh{W},\wh{W},\wh{W}] - M_3[W,W,W]\|_2
} \\
& \leq
\|M_3[\wh{W}-W,\wh{W},\wh{W}]\|_2
+ \|M_3[W,\wh{W}-W,\wh{W}]\|_2
+ \|M_3[W,W,\wh{W}-W]\|_2
\\
& \leq
\|M_3[W,W,W]\|_2
\biggl(
\|(\wh{W}^\t M_2 \wh{W})^{1/2} - I\|_2
\|(\wh{W}^\t M_2 \wh{W})^{1/2}\|_2^2
\\
& \qquad
+ \|(\wh{W}^\t M_2 \wh{W})^{1/2} - I\|_2
\|(\wh{W}^\t M_2 \wh{W})^{1/2}\|_2
+ \|(\wh{W}^\t M_2 \wh{W})^{1/2} - I\|_2
\biggr)
\\
& \leq
\|M_3[W,W,W]\|_2
\biggl(
(1 + (3/2) \Err_{M_2})^2 (3/2) \Err_{M_2}
+ (1 + (3/2) \Err_{M_2}) (3/2) \Err_{M_2}
+ (3/2) \Err_{M_2}
\biggr)
\\
& \leq 6 \|M_3[W,W,W]\|_2 \Err_{M_2}
\leq \frac{6}{\sqrt{w_{\min}}} \Err_{M_2} .
\end{align*}
Thus we can bound $\|\wh{T} - T\|_2$ using the triangle inequality
and the above bound:
\begin{align*}
\|\wh{T} - T\|_2
& = \|\wh{M}_3[\wh{W},\wh{W},\wh{W}] - M_3[W,W,W]\|_2 \\
& \leq \|\wh{M}_3[\wh{W},\wh{W},\wh{W}] - M_3[\wh{W},\wh{W},\wh{W}]\|_2
+ \|M_3[\wh{W},\wh{W},\wh{W}] - M_3[W,W,W]\|_2
\\
& \leq \|\wh{M}_3[\wh{W},\wh{W},\wh{W}] - M_3[\wh{W},\wh{W},\wh{W}]\|_2
+ \frac{6}{\sqrt{w_{\min}}} \Err_{M_2}
.
\qedhere
\end{align*}
\end{proof}

\subsection{Eigendecomposition analysis}

Define
\begin{equation} \label{eq:gamma}
\gamma := \frac1{2\sqrt{w_{\max}} \sqrt{ek}\binom{k+1}{2}}
\end{equation}
where $w_{\max} := \max_{i \in [k]} w_i$.
\begin{lemma}[Random separation] \label{lem:rand-sep}
Let $\theta \in \R^k$ be a random vector distributed uniformly over the
unit sphere in $\R^k$.
Let $Q := \{ e_i - e_j : \{i,j\} \in \binom{k}{2} \} \cup \{ e_i : i \in
[k] \}$.
Then
\begin{equation*}
\Pr\Bigl[ \min_{q \in Q} |\theta^\t W^\t Aq| > \gamma \Bigr]
\geq \frac12
\end{equation*}
where the probability is taken with respect to the distribution of
$\theta$.
\end{lemma}
\begin{proof}
By \reflem{anticonc}, with probability at least $1/2$,
\[
\min_{q \in Q} |\theta^\t W^\t Aq|
> \frac{\min_{q \in Q} \|W^\t Aq\|_2}{2\sqrt{ek}\binom{k+1}{2}}
.
\]
Now fix any $i \neq j$.
By \reflem{whitening} (first claim), $W^\t \Means \diag(w)^{1/2}$ is
orthogonal, so $\|W^\t A(e_i - e_j)\|_2 = \|\diag(w)^{-1/2}(e_i - e_j)\|_2
= \|e_i/\sqrt{w_i} - e_j/\sqrt{w_j}\|_2 = \sqrt{1/w_i + 1/w_j}$.
Similarly, for any $i \in [k]$, $\|W^\t Ae_i\|_2 = \sqrt{1/w_i}$.
Therefore $\min_{q \in Q} \|W^\t Aq\|_2 = \min_{i \in [k]} \sqrt{1/w_i}$.
\end{proof}

Let $\Err_T := \|\wh{T} - T\|_2 / \gamma$.
Let $\theta_1, \theta_2, \dotsc, \theta_t$ be the random unit vectors in
$\R^k$ drawn by the algorithm.
Define $\wh{T}[\theta_{t'}] := \wh{M}_3[\wh{W},\wh{W},\wh{W}\theta_{t'}]$
and $T[\theta_{t'}] := M_3[W,W,W\theta_{t'}]$.
Also, let $\Delta(t') := \min \{ |\lambda_i - \lambda_j| : i \neq j \} \cup
\{ |\lambda_i| : i \in [k] \}$ for the eigenvalues $\{ \lambda_i : i \in
[k] \}$ of $T[\theta_{t'}]$, and let $\wh\Delta(t') := \min \{ |\h\lambda_i
- \h\lambda_j| : i \neq j \} \cup \{ |\lambda_i| : i \in [k] \}$ for the
eigenvalues $\{ \h\lambda_i : i \in [k] \}$ of $\wh{T}[\theta_{t'}]$.

\begin{lemma}[Eigenvalue gap] \label{lem:best-gap}
Pick any $\delta \in (0,1)$.
If $t \geq \log_2(1/\delta)$, then with probability at least $1-\delta$,
the trial $\h\tau := \arg\max_{t' \in [t]} \wh\Delta(t')$ satisfies
\begin{align*}
\wh\Delta(\h\tau) & \geq \gamma - 2\Err_T \gamma
.
\end{align*}
\end{lemma}
\begin{proof}
For each $t' \in [t]$, the eigenvalues of $\h\lambda_1, \h\lambda_2,
\dotsb, \h\lambda_k$ of $\wh{T}[\theta_{t'}]$ (arranged in non-increasing
order) satisfy
\begin{align}
|\h\lambda_i - \h\lambda_j|
& \geq |\lambda_i - \lambda_j| - 2 \|\wh{T}[\theta_{t'}] -
T[\theta_{t'}]\|_2
\geq |\lambda_i - \lambda_j| - 2 \Err_T \gamma
, \quad i \neq j
\label{eq:emp-gap}
\end{align}
where the second inequality follows from Weyl's inequality.
Similarly,
\begin{align}
|\h\lambda_i|
& \geq |\lambda_i| - |\h\lambda_i - \lambda_i|
\geq |\lambda_i| - \Err_T \gamma .
\label{eq:emp-smallest}
\end{align}
By \reflem{tensor-structure}, the eigenvalues of $T[v]$ are $v^\t W^\t
\Means e_i$ for $i \in [k]$.
Thus, by \reflem{rand-sep}, the probability that some $\tau \in [t]$ has
$\Delta(\tau) > \gamma$ is at least $1-\delta$.
In this event, \eqref{eq:emp-gap} implies that trial $\tau$ has
$\wh\Delta(\tau) \geq \Delta(\tau) - 2\Err_T \gamma$, and hence
$\wh\Delta(\h\tau) = \max_{t' \in [t]} \wh\Delta(t') \geq \gamma - 2\Err_T
\gamma$.
\end{proof}

We now just consider the trial $\h\tau$ retained by the algorithm.
Let $\{ (v_i,\lambda_i) : i \in [k] \}$ be the eigenvector/eigenvalue pairs
of $T[\theta_{\h\tau}]$, and let $\{ (\h{v}_i,\h\lambda_i) : i \in [k] \}$
be the eigenvector/eigenvalue pairs of $\wh{T}[\theta_{\h\tau}]$.
\begin{lemma}[Accuracy of eigendecomposition] \label{lem:eig}
Assume the $1-\delta$ probability event in \reflem{best-gap} holds, and
also assume that $\Err_T \leq 1/4$.
Then there exists a permutation $\pi$ on $[k]$ and signs $s_1, s_2, \dotsc,
s_k \in \{\pm1\}$ such that, for all $i \in [k]$, 
\begin{align*}
\|v_i - s_i\h{v}_{\pi(i)}\|_2 & \leq 4\sqrt{2} \Err_T \\
|\lambda_i - \h\lambda_{\pi(i)}| & \leq \Err_T \gamma .
\end{align*}
\end{lemma}
\begin{proof}
To simplify notation, assume the eigenvalues of $\wh{T}[\theta]$ and
$T[\theta]$ are already sorted in non-increasing order.
Observe that for all $i \neq j$,
\begin{align*}
|\h\lambda_i - \lambda_j|
& = |\h\lambda_i - \h\lambda_j + \h\lambda_j - \lambda_j|
\\
& \geq |\h\lambda_i - \h\lambda_j| - |\h\lambda_j - \lambda_j|
\\
& \geq (\gamma - 2\Err_T \gamma) - \Err_T \gamma \\
& \geq \gamma/4
\end{align*}
where the second-to-last inequality follows by the assumption
$\wh\Delta(\hat\tau) \geq \gamma - 2\Err_T \gamma$ and by Weyl's
inequality.
Therefore, the interval of radius $\gamma/4$ surronding each eigenvalue
$\h\lambda_i$ of $\wh{T}[\theta_{\h\tau}]$ contains only one eigenvalue
$\lambda_i$ of $T[\theta_{\h\tau}]$.
By the Davis-Kahan $\sin(\Theta)$ theorem~\citep[Theorem 3.4,
p.~250]{SS90}, we have that
\begin{align*}
\sqrt{1 - (v_i^\t \h{v}_i)^2} & \leq 4\Err_T
.
\end{align*}
Therefore, for $s_i := \sign(v_i^\t\h{v}_i)$,
\begin{align*}
\|v_i - s_i \h{v}_i\|_2^2
& = 2\bigl( 1 - s_i v_i^\t \h{v}_i \bigr)
= 2\bigl( 1 - |v_i^\t \h{v}_i| \bigr)
\leq 2\bigl( 1 - \sqrt{1 - (4\Err_T)^2} \bigr)
\leq 32\Err_T^2 .
\end{align*}
The bound $|\lambda_i - \h\lambda_i| \leq \Err_T \gamma$ follows simply
from Weyl's inequality.
\end{proof}

\subsection{Overall error analysis}

Define
\begin{align*}
\kappa[M_2] & := \sig_1[M_2] / \sig_k[M_2]
, \\
\epsilon_0 & :=
\biggl(
5.5 \Err_{M_2}
+ 7 \Err_T
\biggr) / \sqrt{w_{\min}}
, \\
\epsilon_1 & :=
\biggl(
\frac{1.25\|M_2\|_2^{1/2}\epsilon_0 / \sqrt{w_{\min}}}{\sig_k[M_2]^{1/2}}
+ 2\Err_{M_2}
+ \gamma\sqrt{w_{\min}} \Err_T
\biggr) / \bigl( \gamma \sqrt{w_{\min}} \bigr)
\\
& \hphantom:=
\biggl(
\Bigl( 6.875 \kappa[M_2]^{1/2} + 2 \Bigr) \Err_{M_2}
+ \Bigl( 8.75 \kappa[M_2]^{1/2} + \gamma \sqrt{w_{\min}} \Bigr) \Err_T
\biggr) / \bigl( \gamma \sqrt{w_{\min}} \bigr)
.
\end{align*}

\begin{lemma}[Error bound] \label{lem:error-bound}
Assume the $1-\delta$ probability event of \reflem{best-gap} holds, and
also assume that $\Err_{M_2} \leq 1/3$, $\Err_T \leq 1/4$, and $\epsilon_1
\leq 1/3$.
Then there exists a permutation $\pi$ on $[k]$ such that
\begin{align*}
\|\h\mean_{\pi(i)} - \mean_i\|_2
& \leq
3\|\mean_i\|_2 \epsilon_1 + 2 \|M_2\|_2^{1/2} \epsilon_0
, \quad i \in [k] .
\end{align*}
\end{lemma}
\begin{proof}
To simplify notation, we assume throughout that the permutation $\pi$ from
\reflem{eig} is the identity permutation.
Let $V := [v_1 | v_2 | \dotsb | v_k]$.
We first bound $\wh{B} s_i \h{v}_i - \sqrt{w_i} \mean_i$.
This quantity can be split into two parts: the part in the range of
$\wh{U}$, and the rest.
The part in the range of $\wh{U}$ is bounded as
\begin{align*}
\|\wh{B} s_i \h{v}_i - \wh{U}\wh{U}^\t \sqrt{w_i} \mean_i\|_2
& = \|(\wh{U}^\t \wh{M}_2 \wh{U})^{1/2} s_i \h{v}_i
- \wh{U}^\t \Means \diag(w)^{1/2} e_i\|_2
\\
& = \|(\wh{U}^\t \wh{M}_2 \wh{U})^{1/2} (s_i \h{v}_i - v_i)
+ ((\wh{U}^\t \wh{M}_2 \wh{U})^{1/2}
- \wh{U}^\t \Means \diag(w)^{1/2} V^\t) v_i\|_2
\\
& \leq
\|(\wh{U}^\t \wh{M}_2 \wh{U})^{1/2}\|_2 \|s_i \h{v}_i - v_i\|_2
+ \|(\wh{U}^\t \wh{M}_2 \wh{U})^{1/2}
- \wh{U}^\t \Means \diag(w)^{1/2} V^\t\|_2
\\
& \leq
\|\wh{M}_2\|_2^{1/2} \|s_i \h{v}_i - v_i\|_2
+ \|(\wh{U}^\t \wh{M}_2 \wh{U})^{1/2}
- \wh{U}^\t \Means \diag(w)^{1/2} V^\t\|_2
\\
& \leq
\Bigl( (1 + \Err_{M_2}) \|M_2\|_2 \Bigr)^{1/2} 4\sqrt{2}\Err_T
+ \|(\wh{U}^\t \wh{M}_2 \wh{U})^{1/2}
- \wh{U}^\t \Means \diag(w)^{1/2} V^\t\|_2
\end{align*}
where the third inequality follows from \reflem{svd} and \reflem{eig}.
To bound the second term in the last step, recall that $W = \wh{W}
(\wh{W}^\t M_2 \wh{W})^{-1/2}$ (using \reflem{whitening} to guarantee the
positive definiteness of $\wh{W}^\t M_2 \wh{W}$), so we may write
$\wh{U}^\t \Means V^\t$ as
\begin{align*}
\wh{U}^\t \Means \diag(w)^{1/2} V^\t
& = \wh{U}^\t \Means \diag(w)^{1/2} (W^\t \Means \diag(w)^{1/2})^\t \\
& = \wh{U}^\t M_2 \wh{W} (\wh{W}^\t M_2 \wh{W})^{-1/2}
\\
& =
\Bigl( (\wh{U}^\t M_2 \wh{U})^{1/2}
+ \wh{U}^\t (M_2 - \wh{M}_2) \wh{U} (\wh{U}^\t \wh{M}_2 \wh{U})^{-1/2}
\Bigr)
(\wh{W}^\t M_2 \wh{W})^{-1/2}
.
\end{align*}
Therefore
\begin{align*}
\lefteqn{
\|(\wh{U}^\t \wh{M}_2 \wh{U})^{1/2}
- \wh{U}^\t \Means \diag(w)^{1/2} V^\t\|_2
} \\
& \leq
\|(\wh{U}^\t \wh{M}_2 \wh{U})^{1/2} \Bigl(
I - (\wh{W}^\t M_2 \wh{W})^{-1/2} \Bigr)\|_2
+ 
\|\wh{U}^\t (M_2 - \wh{M}_2) \wh{U}\|_2
\|(\wh{U}^\t \wh{M}_2 \wh{U})^{-1/2}\|_2
\|(\wh{W}^\t M_2 \wh{W})^{-1/2}\|_2
\\
& \leq \|\wh{M}_2\|_2^{1/2} 
\|I - (\wh{W}^\t M_2 \wh{W})^{-1/2}\|_2
+ \frac1{\sig_k[\wh{U}^\t \wh{M}_2 \wh{U}]^{1/2}}
\|M_2 - \wh{M}_2\|_2
\Bigl( 1 + \|I - (\wh{W}^\t M_2 \wh{W})^{-1/2}\|_2 \Bigr)
\\
& \leq (1+\Err_{M_2})^{1/2} \|M_2\|_2^{1/2} (3/2) \Err_{M_2}
+ \frac{(1 + (3/2)\Err_{M_2}) \Err_{M_2} \sig_k[M_2]}
{(1-\Err_{M_2})^{1/2} \sig_k[M_2]^{1/2}}
\\
& = (1+\Err_{M_2})^{1/2} \|M_2\|_2^{1/2} (3/2) \Err_{M_2}
+ \frac{(1 + (3/2)\Err_{M_2})}{(1-\Err_{M_2})^{1/2}}
\sig_k[M_2]^{1/2}\Err_{M_2}
\end{align*}
where the last inequality uses \reflem{svd} and \reflem{whitening}.
Thus
\begin{align*}
\|\wh{B} s_i \h{v}_i - \wh{U}\wh{U}^\t \sqrt{w_i} \mean_i\|_2
& \leq
(1 + \Err_{M_2})^{1/2} \|M_2\|_2^{1/2} 4\sqrt{2}\Err_T
\\
& \qquad
+ (1+\Err_{M_2})^{1/2} \|M_2\|_2^{1/2} (3/2) \Err_{M_2}
+ \frac{(1 + (3/2)\Err_{M_2})}{(1-\Err_{M_2})^{1/2}}
\sig_k[M_2]^{1/2}\Err_{M_2}
.
\end{align*}
Now consider the part not in the range of $\wh{U}$.
This is simply bounded as
\begin{align*}
\|(I - \wh{U} \wh{U}^\t) \sqrt{w_i} UU^\t \mean_i\|_2
& \leq
\|I - \wh{U} \wh{U}^\t UU^\t\|_2 \sqrt{w_i} \|\mean_i\|_2
\\
& \leq (3/2) \Err_{M_2} \sqrt{w_i} \|\mean_i\|_2
\end{align*}
using \reflem{svd}.
Therefore, overall, we have
\begin{align*}
\|\wh{B} \h{v}_i - s_i \sqrt{w_i} \mean_i\|_2
& = \|\wh{B} s_i \h{v}_i - \sqrt{w_i} \mean_i\|_2
\\
& \leq
(1 + \Err_{M_2})^{1/2} \|M_2\|_2^{1/2} 4\sqrt{2}\Err_T
+ (1+\Err_{M_2})^{1/2} \|M_2\|_2^{1/2} (3/2) \Err_{M_2}
\\
& \qquad
+ \frac{(1 + (3/2)\Err_{M_2})}{(1-\Err_{M_2})^{1/2}}
\sig_k[M_2]^{1/2}\Err_{M_2}
+ (3/2) \Err_{M_2} \sqrt{w_i} \|\mean_i\|_2 \leq
\|M_2\|_2^{1/2} \epsilon_0 \sqrt{w_{\min}} .
\end{align*}

Since the actual estimate of $\mean_i$ is $(\h\lambda_i / \theta^\t
\h{v}_i) \wh{B} \h{v}_i$, we need to show that $\theta^\t \h{v}_i$ is
approximately $s_i \sqrt{w_i} \h\lambda_i$.
Indeed,
\begin{align*}
|\theta^\t \h{v}_i - s_i \sqrt{w_i} \h\lambda_i|
& = |\theta^\t \wh{W}^\t (\h{B} \h{v}_i - s_i\sqrt{w_i} \mean_i)
+ s_i\sqrt{w_i} \theta^\t (\wh{W}  - W)^\t \mean_i
+ s_i\sqrt{w_i} (\lambda_i - \h\lambda_i)| \\
& \leq \|\wh{W}\theta\|_2 \|\wh{B}\h{v}_i - s_i\sqrt{w_i}\mean_i\|_2
+ \|(\wh{W} - W)^\t \Means \diag(w)^{1/2} e_i\|_2
+ \sqrt{w_i} |\lambda_i - \h\lambda_i| \\
& \leq \frac{\|M_2\|_2^{1/2} \epsilon_0 \sqrt{w_{\min}}}
{(1-\Err_{M_2})^{1/2} \sig_k[M_2]^{1/2}}
+ (3/2) \sqrt{1 + (3/2)\Err_{M_2}} \Err_{M_2}
+ \sqrt{w_i} \Err_T \gamma \leq \epsilon_1 \gamma\sqrt{w_{\min}}
\end{align*}
where the last inequality uses \reflem{whitening} and \reflem{eig}.
Therefore
\begin{align*}
\lefteqn{
\sqrt{w_i}
\|(\h\lambda_i / \theta^\t \h{v}_i) \wh{B} \h{v}_i - \mean_i\|_2
} \\
& =
\|(\h\lambda_i / \theta^\t \h{v}_i) s_i\sqrt{w_i} \wh{B} \h{v}_i -
s_i\sqrt{w_i} \mean_i\|_2
\\
& \leq
|(\h\lambda_i / \theta^\t \h{v}_i) s_i\sqrt{w_i} - 1| \|\wh{B} \h{v}_i\|_2
+ \|\wh{B} \h{v}_i - s_i\sqrt{w_i} \mean_i\|_2
\\
& \leq
|(\h\lambda_i s_i\sqrt{w_i} / \theta^\t \h{v}_i) - 1|
(\sqrt{w_i} \|\mean_i\|_2 + \|M_2\|_2^{1/2} \epsilon_0 \sqrt{w_{\min}})
+ \|M_2\|_2^{1/2} \epsilon_0 \sqrt{w_{\min}}
\\
& \leq
\frac{|\h\lambda_i s_i\sqrt{w_i} - \theta^\t\h{v}_i|}
{|\h\lambda_i\sqrt{w_i}| - |\theta^\t \h{v}_i - \h\lambda_is_i\sqrt{w_i}|}
(\sqrt{w_i} \|\mean_i\|_2 + \|M_2\|_2^{1/2} \epsilon_0 \sqrt{w_{\min}})
+ \|M_2\|_2^{1/2} \epsilon_0 \sqrt{w_{\min}}
\\
& \leq
\frac{\epsilon_1 \gamma\sqrt{w_{\min}}}
{|\lambda_i\sqrt{w_i}| - |\h\lambda_i - \lambda_i|\sqrt{w_i} - \epsilon_1
\gamma\sqrt{w_{\min}}}
(\sqrt{w_i} \|\mean_i\|_2 + \|M_2\|_2^{1/2} \epsilon_0 \sqrt{w_{\min}})
+ \|M_2\|_2^{1/2} \epsilon_0 \sqrt{w_{\min}}
\\
& \leq
\frac{\epsilon_1 \gamma\sqrt{w_{\min}}}
{\gamma\sqrt{w_i}(1 - \Err_T) - \epsilon_1 \gamma\sqrt{w_{\min}}}
(\sqrt{w_i} \|\mean_i\|_2 + \|M_2\|_2^{1/2} \epsilon_0 \sqrt{w_{\min}})
+ \|M_2\|_2^{1/2} \epsilon_0 \sqrt{w_{\min}}
\\
& =
\frac{\epsilon_1 \gamma}
{\gamma(1 - \Err_T) - \epsilon_1 \gamma}
(\sqrt{w_i} \|\mean_i\|_2 + \|M_2\|_2^{1/2} \epsilon_0 \sqrt{w_{\min}})
+ \|M_2\|_2^{1/2} \epsilon_0 \sqrt{w_{\min}}
\end{align*}
where the fourth inequality uses \reflem{eig}.
We thus conclude that
\begin{align*}
\|\h\mean_i - \mean_i\|_2
= \|(\h\lambda_i / \theta^\t \h{v}_i) \wh{B} \h{v}_i - \mean_i\|_2
& \leq
\frac{\epsilon_1 \gamma}
{\gamma(1 - \Err_T) - \epsilon_1 \gamma}
(\|\mean_i\|_2 + \|M_2\|_2^{1/2} \epsilon_0)
+ \|M_2\|_2^{1/2} \epsilon_0
\\
& \leq
3\epsilon_1 (\|\mean_i\|_2 + \|M_2\|_2^{1/2} \epsilon_0)
+ \|M_2\|_2^{1/2} \epsilon_0
\\
& \leq 3\|\mean_i\|_2 \epsilon_1 + 2\|M_2\|_2^{1/2} \epsilon_0
.
\qedhere
\end{align*}
\end{proof}

We can now prove \refthm{finite}, stated below with the explicit polynomial
sample complexity bound (up to constants).

\begin{thm2}{finite}[Finite sample bound]
There exists a constant $C > 0$ such that the following holds.
Let $b_{\max} := \max_{i \in [k]} \|\mean_i\|_2$.
Pick any $\veps, \delta \in (0,1)$.
Suppose the sample size $2n$ satisfies
\begin{align*}
n & \geq
C \cdot \frac{d + \log(k/\delta)}{w_{\min}}
\cdot
\Biggl(
\biggl[
\frac{\kappa[M_2]^{1/2} (\sigma^2 + b_{\max}^2)}
{\gamma^2w_{\min}\sig_k[M_2]\veps}
\biggr]^2
+
\biggl[
\frac{\kappa[M_2]^{1/2} (\sigma^2 + b_{\max}^2)}
{\gamma^2w_{\min}\sig_k[M_2]\veps}
\biggr]
\Biggr)
\\
& \qquad
+ C \cdot \frac{(k + \log(k/\delta))^3}{w_{\min}}
\cdot \biggl[
\frac{\kappa[M_2]^{1/2} \sigma^3}
{\gamma^2\sqrt{w_{\min}} \sig_k[M_2]^{3/2} \veps}
\biggr]^2
\\
& \qquad
+ C \cdot \frac{k + \log(k/\delta)}{w_{\min}}
\cdot
\Biggl(
\biggl[
\frac{\kappa[M_2]^{1/2} \sigma^2}
{\gamma^2 w_{\min} \sig_k[M_2] \veps}
\biggr]^2
+ \biggl[
\frac{\kappa[M_2]^{1/2} \sigma^2}
{\gamma^2 w_{\min} \sig_k[M_2] \veps}
\biggr]
+ \biggl[
\frac{\kappa[M_2]^{1/2} \sigma}
{\gamma^2 w_{\min}^{3/2} \sig_k[M_2]^{1/2} \veps}
\biggr]^2
\\
& \qquad\qquad
+ \biggl[
\frac{\kappa[M_2]^{1/2}}{\gamma^2\sqrt{w_{\min}}\veps}
\cdot \frac{\sigma^2}{\sig_k[M_2]^{1/2}}
\cdot \max\Bigl\{ 1, \sigma^2 / \sig_k[M_2] \Bigr\}
\biggr]^2
\Biggr)
\\
& \qquad
+ C \cdot \frac{k\log(1/\delta)}{w_{\min}}
\Biggl(
 \biggl[
\frac{\kappa[M_2]^{1/2}}{\gamma^2\sqrt{w_{\min}}\veps}
\cdot \max\Bigl\{ 1, \sigma^2 / \sig_k[M_2] \Bigr\}
\biggr]^2
+ \biggl[
\frac{\kappa[M_2]^{1/2}}
{\gamma^2 w_{\min}^2 \veps}
\biggr]^2
+ \biggl[
\frac{\kappa[M_2]^{1/2} \sigma^2}
{\gamma^2 w_{\min} \sig_k[M_2] \veps}
\biggr]^2
\Biggr)
\end{align*}
where
\begin{equation*}
\gamma = \frac1{2\sqrt{w_{\max}} \sqrt{ek}\binom{k+1}{2}}
\end{equation*}
(as defined in \reflem{rand-sep}).
Then with probability at least $1-3\delta$ over the random sample and the
internal randomness of the algorithm, there exists a permutation $\pi$ on
$[k]$ such that
\begin{align*}
\|\h\mean_{\pi(i)} - \mean_i\|_2
& \leq \Bigl( \|\mean_i\|_2 + \|M_2\|_2^{1/2} \Bigr) \veps
\end{align*}
for all $i \in [k]$.
\end{thm2}
\begin{proof}
Throughout, $C_1, c_1, C_2, c_2, \dotsc$ will represent absolute positive
constants.
First, observe that the sample size bound
\begin{align*}
n & \geq C \cdot k\log(1/\delta)
\end{align*}
and \reflem{w-conc} ensure that $\Err_w \leq 1$ (where $\Err_w$ is defined
in \reflem{moments}).
Therefore, from \reflem{conditional-moments-conc} and \reflem{moments}
(together with union bounds), with probability at least $1-\delta$,
\begin{align*}
\|\h\mean - \mean\|_2
& \leq
C_1 \sigma \sqrt{\frac{d + \log(k/\delta)}{w_{\min}n}}
+ C_1 b_{\max} \sqrt{\frac{k \log(1/\delta)}{n}}
,
\\
\|\wh\Moment_2 - \Moment_2\|_2
& \leq
C_1 \Biggl(
\sigma^2 \sqrt{\frac{d + \log(k/\delta)}{w_{\min}n}}
+ \sigma^2 \frac{d + \log(k/\delta)}{w_{\min}n}
+ \sigma b_{\max} \sqrt{\frac{d + \log(k/\delta)}{w_{\min}n}}
\Biggr)
\\
& \qquad
+ C_1 \Bigl( b_{\max}^2 + \sigma^2 \Bigr)
\sqrt{\frac{k\log(1/\delta)}{n}}
\\
& \leq
C_1 \Biggl(
1.7\Bigl( \sigma^2 + b_{\max}^2 \Bigr)
\sqrt{\frac{d + \log(k/\delta)}{w_{\min}n}}
+ \sigma^2 \frac{d + \log(k/\delta)}{w_{\min}n}
\Biggr)
.
\end{align*}
Therefore, by \reflem{M2},
\begin{align*}
\max\{|\h\sigma^2 - \sigma^2|, \
\|\wh{M}_2 - M_2\|_2\}
& \leq
4C_1 \Biggl(
1.7\Bigl( \sigma^2 + b_{\max}^2 \Bigr)
\sqrt{\frac{d + \log(k/\delta)}{w_{\min}n}}
+ \sigma^2 \frac{d + \log(k/\delta)}{w_{\min}n}
\Biggr)
\\
& \qquad
+ 4C_1 \|\mean\|_2 \Biggl(
\sigma \sqrt{\frac{d + \log(k/\delta)}{w_{\min}n}}
+ b_{\max} \sqrt{\frac{k \log(1/\delta)}{n}}
\Biggr)
\\
& \qquad
+ 2C_1^2
\Biggl(
\sigma \sqrt{\frac{d + \log(k/\delta)}{w_{\min}n}}
+ b_{\max} \sqrt{\frac{k \log(1/\delta)}{n}}
\Biggr)^2
\\
& \leq
C_2 (\sigma^2 + b_{\max}^2)
\Biggl(
\sqrt{\frac{d + \log(k/\delta)}{w_{\min}n}}
+ \frac{d + \log(k/\delta)}{w_{\min}n}
\Biggr)
.
\end{align*}
The sample size bound
\begin{equation*}
n \geq
C \cdot \frac{d + \log(k/\delta)}{w_{\min}}
\cdot
\Biggl(
\biggl[
\frac{\kappa[M_2]^{1/2} (\sigma^2 + b_{\max}^2)}
{\gamma^2w_{\min}\sig_k[M_2]\veps}
\biggr]^2
+
\biggl[
\frac{\kappa[M_2]^{1/2} (\sigma^2 + b_{\max}^2)}
{\gamma^2w_{\min}\sig_k[M_2]\veps}
\biggr]
\Biggr)
,
\end{equation*}
ensures that
\begin{equation} \label{eq:bound2}
\max\Bigl\{ \frac{|\h\sigma^2 - \sigma^2|}{\sig_k[M_2]}, \Err_{M_2}
\Bigr\}
\leq c_1 \frac{\gamma^2w_{\min}}{\kappa[M_2]^{1/2}}
\veps
\leq 1/3
.
\end{equation}
Now condition on the above event and take $\wh{W}$ as given.
By \reflem{whitening},
\begin{align*}
\|\wh{W}\|_2 & \leq \sqrt{1.5 / \sig_k[M_2]} , \\
\max_{i \in [k]}
\|\wh{W}^\t\mean_i\|_2 & \leq \|\wh{W}^\t \Means\diag(w)^{1/2}\|_2 /
\sqrt{w_{\min}} \leq \sqrt{1.5 / w_{\min}} , \\
\max_{i \in [k]}
\|\wh{W}^\t\Moment_{2,i}\wh{W}\|_2
& \leq 1.5 / \sig_k[M_2] + 1.5 / w_{\min} , \\
\max_{i \in [k]}
\|\Moment_{3,i}[\wh{W},\wh{W},\wh{W}]\|_2
& \leq (1.5 / w_{\min})^{3/2} + 3 \sigma^2 \sqrt{1.5 / w_{\min}} 1.5 /
\sig_k[M_2] .
\end{align*}
Therefore, \reflem{conditional-moments-conc} and \reflem{moments} imply
that with probability at least $1-\delta$,
\begin{align*}
\|\wh{W}^\t (\underline{\h\mean} - \mean)\|_2
& \leq C_3
\frac{\sigma}{\sig_k[M_2]^{1/2}}
\sqrt{\frac{k + \log(k/\delta)}{w_{\min}n}}
+ C_3 \frac1{\sqrt{w_{\min}}} \sqrt{\frac{k \log(1/\delta)}{n}}
,
\\
\|\underline{\wh\Moment_3}[\wh{W},\wh{W},\wh{W}]
- \Moment_3[\wh{W},\wh{W},\wh{W}]\|_2
& \leq C_3
\frac{\sigma^3}{\sig_k[M_2]^{3/2}} \sqrt{\frac{(k +
\log(k/\delta))^3}{w_{\min}n}}
\\
& \qquad
+ C_3 \frac{\sigma^2}{\sqrt{w_{\min}}\sig_k[M_2]}
\Biggl[
\sqrt{\frac{k + \log(k/\delta)}{w_{\min}n}}
+ \frac{k + \log(k/\delta)}{w_{\min}n}
\Biggr]
\\
& \qquad
+ C_3 \frac{\sigma}{w_{\min}\sig_k[M_2]^{1/2}}
\sqrt{\frac{k + \log(k/\delta)}{w_{\min}n}}
\\
& \qquad
+ C_3 \Biggl( \frac1{w_{\min}^{3/2}} +
\frac{\sigma^2}{\sqrt{w_{\min}}\sig_k[M_2]} \Biggr)
\sqrt{\frac{k \log(1/\delta)}{n}}
.
\end{align*}
Therefore, by \reflem{M3} and \reflem{tensor},
\begin{align}
\|\wh{T}-T\|_2
& \leq \|\underline{\wh\Moment_3}[\wh{W},\wh{W},\wh{W}]
- \Moment_3[\wh{W},\wh{W},\wh{W}]\|_2
\nonumber \\
& \qquad
+ \frac{4.5}{\sig_k[M_2]} \Bigl( \|\wh{W}^\t(\underline{\h\mean} -
\mean)\|_2 + \sqrt{1.5/w_{\min}} \Bigr) |\h\sigma^2 - \sigma|
+ \frac{1.5\sigma^2}{\sig_k[M_2]}
\|\wh{W}^\t(\underline{\h\mean} - \mean)\|_2
.
\label{eq:T}
\end{align}
The sample size bound
\begin{align*}
n & \geq C \cdot \frac{k + \log(k/\delta)}{w_{\min}}
\cdot \biggl[
\frac{\kappa[M_2]^{1/2}}{\gamma^2\sqrt{w_{\min}}\veps}
\cdot \frac{\sigma^2}{\sig_k[M_2]^{1/2}}
\cdot \max\Bigl\{ 1, \sigma^2 / \sig_k[M_2] \Bigr\}
\biggr]^2 \\
& \qquad + C \cdot \frac{k\log(1/\delta)}{w_{\min}}
\cdot \biggl[
\frac{\kappa[M_2]^{1/2}}{\gamma^2\sqrt{w_{\min}}\veps}
\cdot \max\Bigl\{ 1, \sigma^2 / \sig_k[M_2] \Bigr\}
\biggr]^2
\end{align*}
ensures
\begin{equation} \label{eq:bound1}
\max\Bigl\{ 1, \sigma^2 / \sig_k[M_2]^{1/2} \Bigr\}
\frac{\|\wh{W}^\t (\underline{\h\mean} - \mean)\|_2}{\gamma} \leq c_2
\frac{\gamma\sqrt{w_{\min}}}{\kappa[M_2]^{1/2}} \veps \leq 1
.
\end{equation}
Furthermore, the sample size bound
\begin{align*}
n & \geq C \cdot \frac{(k + \log(k/\delta))^3}{w_{\min}}
\cdot \biggl[
\frac{\kappa[M_2]^{1/2} \sigma^3}
{\gamma^2\sqrt{w_{\min}} \sig_k[M_2]^{3/2} \veps}
\biggr]^2
\\
& \qquad
+ C \cdot \frac{k + \log(k/\delta)}{w_{\min}}
\cdot \Biggl(
\biggl[
\frac{\kappa[M_2]^{1/2} \sigma^2}
{\gamma^2 w_{\min} \sig_k[M_2] \veps}
\biggr]^2
+ \biggl[
\frac{\kappa[M_2]^{1/2} \sigma^2}
{\gamma^2 w_{\min} \sig_k[M_2] \veps}
\biggr]
\Biggr)
\\
& \qquad
+ C \cdot \frac{k + \log(k/\delta)}{w_{\min}}
\cdot
\biggl[
\frac{\kappa[M_2]^{1/2} \sigma}
{\gamma^2 w_{\min}^{3/2} \sig_k[M_2]^{1/2} \veps}
\biggr]^2
\\
& \qquad
+ C \cdot k\log(1/\delta)
\cdot
\Biggl(
\biggl[
\frac{\kappa[M_2]^{1/2}}
{\gamma^2 w_{\min}^2 \veps}
\biggr]^2
+ \biggl[
\frac{\kappa[M_2]^{1/2} \sigma^2}
{\gamma^2 w_{\min} \sig_k[M_2] \veps}
\biggr]^2
\Biggr)
\end{align*}
ensures
\begin{equation} \label{eq:bound3}
\frac{\|\underline{\wh\Moment_3}[\wh{W},\wh{W},\wh{W}]
- \Moment_3[\wh{W},\wh{W},\wh{W}]\|_2}{\gamma}
\leq c_2
\frac{\gamma\sqrt{w_{\min}}}{\kappa[M_2]^{1/2}} \veps .
\end{equation}
Using the inequalities \eqref{eq:bound1}, \eqref{eq:bound3}, and
\eqref{eq:bound2} with \eqref{eq:T} gives
\begin{equation} \label{eq:boundT}
\Err_T = \frac{\|\wh{T} - T\|_2}{\gamma}
\leq 
c_3 \frac{\gamma\sqrt{w_{\min}}}{\kappa[M_2]^{1/2}} \veps .
\end{equation}
Since $t = \log_2(1/\delta)$, the inequality from \reflem{best-gap} holds
with probability at least $1-\delta$ over the internal randomness of the
algorithm.
Therefore, using \eqref{eq:bound2} and \eqref{eq:boundT} with
\reflem{error-bound} gives in this event:
\begin{align*}
\|\h\mean_{\pi(i)} - \mean_i\|_2
& \leq
C_4 \cdot \|\mean_i\|_2
\cdot \Bigl( \Err_{M_2} + \Err_T \Bigr)
\cdot \frac{\kappa[M_2]^{1/2}}{\gamma \sqrt{w_{\min}}}
\\
& \qquad
+ C_4 \cdot \|M_2\|_2^{1/2}
\cdot \Bigl( \Err_{M_2} + \Err_T \Bigr)
\cdot \frac{1}{\sqrt{w_{\min}}}
\\
& \leq \Bigl( \|\mean_i\|_2 + \|M_2\|_2^{1/2} \Bigr) \veps
\end{align*}
for all $i \in [k]$, for some permutation $\pi$ on $[k]$.
\end{proof}

\section{Probability tail inequalities}

We recall and derive some probability tail inequalities used in the
analysis.

\begin{lemma}[\citealp{DG03,AHK12}] \label{lem:anticonc}
Pick any $\delta \in (0,1)$, matrix $X \in \R^{p \times p}$, and finite
subset $Q \subseteq \R^p$.
If $\theta \in \R^p$ be a random vector distributed uniformly over
the unit sphere in $\R^p$, then
\begin{equation*}
\Pr\biggl[ \min_{q \in Q} |\theta^\t Xq| > \frac{\min_{q \in Q} \|Xq\|_2
\cdot \delta}{\sqrt{ep} |Q|} \biggr] \geq 1-\delta .
\end{equation*}
\end{lemma}

\begin{lemma}[\citealp{LauMas00}]
\label{lem:chi2}
Let $z_1^2, z_2^2, \dotsc, z_m^2$ be i.i.d.~$\chi^2$ random variables, each
with one degree of freedom.
Then for any $\delta \in (0,1)$,
\begin{equation*}
\Pr\biggl[ \sum_{i=1}^m z_i^2
> m + 2\sqrt{m\ln(1/\delta)} + 2\ln(1/\delta)
\biggr]
\leq \delta .
\end{equation*}
\end{lemma}

\begin{lemma}[\citealp{LPRTJ05,HKZ12matrix}]
\label{lem:normal-matrix}
Let $y_1, y_2, \dotsc, y_m$ be i.i.d.~$N(0,I)$ random vectors in $\R^p$.
Then for any $\eps_0 \in (1,1/2)$ and $\delta \in (0,1)$,
\begin{equation*}
\Pr\Biggl[ \biggl\| \frac1m \sum_{i=1}^m y_i y_i^\t - I
\biggr\|_2 > \frac1{1-2\eps_0}
\Biggl( \sqrt{\frac{32 \ln((1+2/\eps_0)^p/\delta)}{m}}
+ \frac{2 \ln((1+2/\eps_0)^p/\delta)}{m} \Biggr)
\Biggr] \leq \delta .
\end{equation*}
\end{lemma}

\begin{lemma}[Sums of cubes of normal random variables]
\label{lem:normal-cubed}
Let $z_1, z_2, \dotsc, z_m$ be i.i.d.~$N(0,1)$ random variables.
Then for any $\delta \in (0,1)$,
\begin{equation*}
\Pr\biggl[ \biggl| \sum_{i=1}^m z_i^3 \biggr|
> \sqrt{27e^3 m \lceil\ln(1/\delta)\rceil^3}
\biggr] \leq \delta .
\end{equation*}
\end{lemma}
\begin{proof}
We use Markov's inequality via the $p$-th moment to derive the tail
inequality.
Pick some even integer $p \geq 2$, and observe that
\begin{align*}
\E\biggl[ \biggl( \sum_{i=1}^m z_i^3 \biggr)^p \biggr]
& = \sum_{i_1, i_2, \dotsc, i_p \in [m]} \E\biggl[ \prod_{j=1}^p z_{i_j}^3
\biggr] .
\end{align*}
By the independence of the $z_i$'s, a term in the summation is zero if any
index $i \in [m]$ is selected an odd number of times (\emph{i.e.}, $|\{ j
\in [p] : i_j = i \}|$ is odd, for any $i \in [m]$).
Therefore the summation can be written as
\begin{align*}
\sum_{p_1 + \dotsb + p_m = p/2} \E\biggl[ \prod_{i=1}^m z_i^{6p_i} \biggr]
& = \sum_{p_1 + \dotsb + p_m = p/2} \prod_{i=1}^m \E\bigl[ z_i^{6p_i} \bigr]
= \sum_{p_1 + \dotsb + p_m = p/2} \prod_{i=1}^m (6p_i - 1)!!
\end{align*}
where the summations are over non-negative integers $p_1, p_2, \dotsc, p_m$
that sum to $p/2$, and $n!!$ is the product of all odd integers between $1$
and $n$; the last step uses the well-known fact that $\E[z^k] = (k-1)!!$
for a standard normal random variable $z$.
As $p_i \leq p/2$ for each $i \in [m]$, the product can be crudely bounded
by $(3p)^{3p/2}$, and hence the sum is bounded by
\begin{align*}
(3p)^{3p/2} \binom{p/2 + m - 1}{p/2}
& \leq (27p^3)^{p/2}
\biggl( \frac{(p/2 + m - 1)e}{p/2} \biggr)^{p/2}
\leq (27ep^3m)^{p/2} .
\end{align*}
By Markov's inequality, for any $t > 0$,
\begin{align*}
\Pr\biggl[ \biggl| \sum_{i=1}^m z_i^3 \biggr|
> t \biggr]
& \leq t^{-p} \E\biggl[ \biggl( \sum_{i=1}^m z_i^3 \biggr)^p \biggr]
\leq \biggl( \frac{\sqrt{27ep^3m}}{t} \biggr)^p .
\end{align*}
The bound is at most $\delta$ for $t := e\sqrt{27e
m\lceil\ln(1/\delta)\rceil^3}$ and $p := \lceil \ln(1/\delta) \rceil$.
\end{proof}

\begin{lemma}[Third-order tensor of normal random vectors]
\label{lem:normal-tensor}
Let $y_1, y_2, \dotsc, y_m$ be i.i.d.~$N(0,I)$ random vectors in $\R^p$.
Then for any $\eps_0 \in (1,1/3)$ and $\delta \in (0,1)$,
\begin{equation*}
\Pr\Biggl[ \biggl\| \frac1m \sum_{i=1}^m y_i \otimes y_i \otimes y_i
\biggr\|_2 > \frac1{1-3\eps_0}
\sqrt{\frac{27e^3\lceil\ln((1+2/\eps_0)^p/\delta)\rceil^3}{m}}
\Biggr] \leq \delta .
\end{equation*}
\end{lemma}
\begin{proof}
We follow the covering approach of \citet{LPRTJ05}.
Let $Q \subseteq \{x \in \R^p : \|x\|_2 = 1\}$ be an $\eps_0$-cover of the
unit sphere in $\R^p$ of cardinality at most $(1+2/\eps_0)^p$, which can be
shown to exist by a standard volume argument \citep{Pisier89}.
Let $Y := m^{-1} \sum_{i=1}^m y_i \otimes y_i \otimes y_i$ and $\eps :=
(27e^3\lceil\ln(|Q|/\delta)\rceil^3/m)^{1/2}$.
Since $y_i^\t q$ is distributed as $N(0,1)$ for any $q \in Q$, it follows
from \reflem{normal-cubed} and a union bound that $\Pr[ \exists q \in Q
\centerdot |Y[q,q,q]| > \eps ] \leq \delta$.
Henceforth we assume $\forall q \in Q \centerdot |Y[q,q,q] \leq \eps$.
Now pick any unit vector $u$ such that $|Y[u,u,u]|$ is maximized
(\emph{i.e.}, $\|Y\|_2 = |Y[u,u,u]|$), choose $q \in Q$ such that $\|q -
u\|_2 \leq \eps_0$, and set $\Delta := u - q$ and $\bar\Delta := \Delta /
\|\Delta\|_2$.
Then
\begin{align*}
\|Y\|_2
& = |Y[u,u,u]| \\
& = |Y[\Delta,u,u] + Y[q,\Delta,u] + Y[q,q,\Delta] + Y[q,q,q]| \\
& \leq \eps_0 (Y[\bar\Delta,u,u] + Y[q,\bar\Delta,u] +
Y[q,q,\bar\Delta]) + \eps
\end{align*}
by the triangle inequality and facts $\|\Delta\|_2 \leq \eps_0$ and
$|Y[q,q,q]| \leq \eps$.
Since $Y$ has the form $Y = \sum_{j=1}^m \tl{y}_j \otimes \tl{y}_j \otimes
\tl{y}_j$ for vectors $\tl{y}_j := m^{-1/3} y_j \in \R^m$, it follows that
\begin{align*}
\sup_{\|u\|_2 = \|v\|_2 = \|w\|_2 = 1} |Y[u,v,w]|
& = \sup_{\|u\|_2 = \|v\|_2 = \|w\|_2 = 1}
\biggl|\sum_{j=1}^r (u^\t \tl{y}_j) (v^\t \tl{y}_j) (w^\t \tl{y}_j) \biggr|
\\
& = \sup_{\|u\|_2 = \|v\|_2 = \|w\|_2 = 1}
\Bigl| u^\t \tl{Y} \diag(w^\t \tl{Y}) \tl{Y}^\t v \Bigr| \\
& = \sup_{\|w\|_2 = 1} \|\tl{Y} \diag(w^\t \tl{Y}) \tl{Y}^\t \|_2 \\
& = \sup_{\|u\|_2 = \|w\|_2 = 1} |Y[u,u,w]|
\end{align*}
where $\tl{Y} = [\tl{y}_1 | \tl{y}_2 | \dotsb | \tl{y}_m] \in \R^{p \times
m}$---\emph{i.e.}, we can take the unit vectors $u$ and $v$ achieving
$|Y[u,v,w]| = \|Y\|_2$ to be the same.
By symmetry, $\sup_{\|u\|_2 = 1} |Y[u,u,u]| = \|Y\|_2$.
Therefore
\begin{align*}
\|Y\|_2
& \leq \eps_0 (Y[\bar\Delta,u,u] + Y[q,\bar\Delta,u] +
Y[q,q,\bar\Delta]) + \eps
\leq 3\eps_0 \|Y\|_2 + \eps
\end{align*}
which implies $\|Y\|_2 \leq \eps / (1-3\eps_0)$.
This proves that
\[ \Pr[ \|Y\|_2 \leq \eps/(1-3\eps_0) ] \geq \Pr[ \forall q
\in Q \centerdot |Y[q,q,q]| \leq \eps ] \geq 1-\delta \] as required.
\end{proof}

\end{document}